\documentclass{article}
\usepackage{aaai19}
\usepackage{lipsum}
\usepackage[utf8]{inputenc} 
\usepackage{url}            
\usepackage{booktabs}       
\usepackage{amsfonts}       
\usepackage{nicefrac}       
\usepackage{microtype}      

\usepackage{latexsym, mathrsfs}
\usepackage{graphicx}
\usepackage[usenames,dvipsnames,svgnames,table]{xcolor}    
\usepackage[linesnumbered,ruled,procnumbered]{algorithm2e}
\usepackage{times}
\usepackage{paralist}
\usepackage{multirow}
\usepackage{xspace}
\usepackage{subcaption} 
\usepackage{hhline}
\usepackage{times}
\usepackage[noend]{algpseudocode}
\usepackage{floatrow}
\usepackage{appendix}
\usepackage{amsmath}
\usepackage{amsthm}
\usepackage{amssymb}

\newcommand{\stitle}[1]{\noindent\textbf{#1}}

\newtheorem{assumption}{\textbf{Assumption}}
\newtheorem{theorem}{\textbf{Theorem}}
\newtheorem{corollary}{{Corollary}}

\newtheorem{problem}{\textbf{Problem}}

\newcommand{\dd}{\mathcal{D}\xspace}
\newcommand{\hh}{\mathcal{H}\xspace}
\newcommand{\ff}{\mathcal{F}\xspace}
\newcommand{\cc}{\mathcal{C}\xspace}
\newcommand{\calT}{\mathcal{T}\xspace}
\newcommand{\calA}{\mathcal{A}\xspace}

\newcommand{\eat}[1]{}
\newcommand{\techreport}[1]{}

\newcommand{\aml}{{\sc Abc}\xspace}

\newcommand{\gucb}{{\sc Geo-Ucb}\xspace}
\newcommand{\grad}{{\sc GradientCI}\xspace}

\newcommand{\ucb}{{\sc Ucb}\xspace}
\newcommand{\plateau}{{\sc CIEstimator}\xspace}
\newcommand{\scheduler}{{\sc Scheduler}\xspace}

\newcommand{\techreporttext}[1]{}

\newcommand{\new}{}
\newcommand{\del}[1]{}

\newenvironment{denselist}{
    \begin{list}{\small{$\bullet$}}%
    {
    }
    }%
    {\end{list}}

\begin{document}
\title{\aml: Efficient Selection of Machine Learning Configuration on Large Dataset}

\author{
Silu Huang$^1$\thanks{Work done while visiting Microsoft Research} ~~ Chi Wang$^2$ ~~~ Bolin Ding$^3$\thanks{Work done while working in Microsoft Research}  ~~~~ Surajit Chaudhuri$^2$ \\
$^1$University of Illinois, Urbana-Champaign, IL \\
~~~~ $^2$Microsoft Research, Redmond, WA \\
~~~~~~ $^3$Alibaba Group, Bellevue, WA \\
shuang86@illinois.edu,~~ \{wang.chi, surajitc\}@microsoft.com, ~~bolin.ding@alibaba-inc.com \\
}

\maketitle

\begin{abstract}
A machine learning configuration refers to a combination of preprocessor, learner, and hyperparameters. Given a set of configurations and a large dataset randomly split into training and testing set, we study how to efficiently select the best configuration with approximately the highest testing accuracy when trained from the training set. To guarantee small accuracy loss, we develop a solution using confidence interval (CI)-based progressive sampling and pruning strategy. Compared to using full data to find the exact best configuration, our solution achieves more than two orders of magnitude speedup, while the returned top configuration has identical or close test accuracy. 
\end{abstract}



\section{Introduction}\label{sec:intro}

Increasing the productivity of data scientists has been a target for many machine learning service providers, such as Azure ML, DataRobot, Google Cloud ML, and AWS ML.
For a new predictive task, a data scientist usually spends a vast amount of time to train a good ML solution. A proper {\em configuration}, i.e., the combination of preprocessor (e.g., feature engineering), learner (i.e., training algorithm) and the associated hyperparameters, is critical to achieving good performance. 
It often takes tens or hundreds of trials to select a suitable configuration. 

There are AutoML tools like auto-sklearn~\cite{NIPS:eurer2015} to automate these trials, and output a configuration with highest evaluated performance. 
However, both the manual and AutoML approaches have become increasingly inefficient as the available ML data volume grows to millions or more. Even the trial for a single configuration can take hours or days for such large-scale datasets. Motivated by this efficiency issue, we propose a module called {\em approximate best configuration} (\aml). Given a set of configurations, it outputs the approximate best configuration, such that the accuracy loss to the best configuration is below a threshold. Our goal is to {\em efficiently} select the approximate best configuration. 

The intuition behind \aml is that the ML model trained over a sampled dataset can be used to approximate the model trained over the full dataset. However, 
the optimal sample size to determine the best configuration up to an accuracy loss threshold is unknown. We develop a novel confidence interval (CI)-based progressive sampling and pruning solution, by addressing two questions: 
\begin{inparaenum}[\itshape (a)\upshape]
\item {\em CI estimator:} given a sampled training dataset, how to estimate the confidence interval of a configuration's real performance with full training data? 
\item {\em scheduler:} as the optimal sample size is unknown \emph{a priori}, how to allocate  appropriate sample size for each configuration? 
\end{inparaenum}

Our contributions are summarized as the following.
\begin{denselist}
  \item We develop an \aml framework using progressive sampling and CI-based pruning. It ensures finding an approximate best configuration while reducing the running time. 
  \item We present and prove bounds for the real test accuracy when the ML model is trained using full data, based on the model trained with sampled data. 
  \item Within \aml, we design an approximately optimal scheduling scheme based on the confidence interval, 
for allocating sample size among different configurations. 
  \item We conduct experiments with large datasets. We demonstrate that our \aml solution is tens to hundreds of times faster, while returning top configurations with no more than 1\% accuracy loss. 
\end{denselist}

\section{Problem Formulation} \label{sec:prob}
\stitle{Notions and Notations.} 
In this paper, we focus on classification tasks with a large set of labeled data $\dd$. 
In order for reliable evaluation of a trained classifier, data scientists usually split the available data {\em randomly} into training and testing set $\dd_{tr}$ and $\dd_{te}$.
After that, they specify a number of configurations of the ML workflow and try to select the best configuration. Let $\cc$ be the candidate configuration set and $C_i$ be the $i^{th}$ configuration in $\cc$. We further let $n$ be the number of configurations, i.e., $n=|\cc|$.
Using terminology from learning theory, each configuration $C_i$ defines a {\em hypothesis space} $\hh_i$, where each {\em hypothesis} $H\in \hh_i$ is a possible classifier trained under this configuration. 
Given a training dataset $\dd_{tr}$, the learner in $C_i$ will output a hypothesis $H_{tr}^i\in \hh_i$ as the trained classifier. The quality of the classifier is measured against the heldout testing data $\dd_{te}$.
In this paper, we focus on {\em accuracy} as the quality metric. We denote the accuracy of hypothesis $H\in \hh_i$ on dataset $D$ as $\calA(H,D)$. In particular, given a configuration $C_i$, we define its \emph{real test accuracy} as $\calA_i = \calA(H_{tr}^i, \dd_{te})$.

\begin{table*}[t]
    \caption{Notations}
    \begin{center}
    \begin{tabular}{ |c|c||c|c| } 
     \hline
    $\cc$ & configuration set & $C_i$ & one configuration in $\cc$ \\
     \hline
     $\hh_i$ & hypothesis class of $C_i$ & $H_{tr}^i$ & hypothesis trained with $\dd_{tr}$ under $C_i$ \\
    \hline
    $\calA_i$ & real test accuracy for $C_i\in \cc$ & $\dd_{te}$ & test dataset\\
    \hline
     $\dd$ & full dataset & $\dd_{tr}$ & training dataset \\ 
      \hline
     $S_{tr}$ & sample set from $\dd_{tr}$ & $S_{te}$ & sample set from $\dd_{te}$ \\
     \hline
     $t_i$ & total running time for $C_i$ & $l_i$ & lower bound of $\calA_i$ \\
     \hline
     $u_i$ & upper bound of $\calA_i$ & $n$ & $n = |\cc| $\\ 
     \hline
    \end{tabular}
    \end{center}
    \end{table*}
    
\stitle{Problem Definition.} 
A standard practice to select the best configuration from a configuration set $\cc$ is to train with each configuration using full training data, and then pick the one with the highest test accuracy, i.e., ${i^*} = \arg\max_{i\in [n]} \calA_i$. Note that an implicit assumption made here is that the returned classifier with full training data has equal or higher test accuracy than the classifier trained with sampled training data. We {\new call this \emph{exploitativeness} assumption} and follow it in this paper.
From a user's perspective, if there are multiple configurations with nearly identical highest real test accuracy, then it would suffice to return any of them as the best configuration. So we introduce a new problem {\em approximate best configuration} selection, as formalized in Problem~\ref{prob:bestIdentify}.

\begin{problem}[Approximate Best Configuration Selection]\label{prob:bestIdentify}
Given a configuration candidate set $\cc$ and an accuracy loss tolerance $\epsilon$, select a configuration $C_{i'}$ whose real test accuracy is within $\epsilon$ away from that of the best configuration $C_{i^*}$, i.e., $\calA_{i^*} - \calA_{i'} \leq \epsilon$, and minimize the total running time.
\end{problem}

\section{CI-based Framework}\label{sec:framework}
Before introducing our framework, we first describe some insights based on simple observations.
We experiment on the {\em FlightDelay} dataset with five learners (as five configurations).
Readers can refer to Table~\ref{table:dataset} for detailed statistics of this dataset. The {\em learning curve} for each configuration is depicted in Figure~\ref{fig:obs}, where x-axis is the training sample size in log-scale and y-axis is the test accuracy on $\dd_{te}$.
In general, the test accuracy approaches the real test accuracy with the increase of the training sample size. When the sample size is large enough ($\ge$2M), the configuration with the highest test accuracy is LightGBM -- the true best configuration. If we knew it before the experiment, we could use a fraction of the training data to select the right configuration.

\begin{table*}[t]
  \begin{center}
  \small
    \begin{tabular}{rrrl}
    \hline
     Name & $|\dd|$ & $|\ff|$ & Origin\\ \hline
     Twitter & 1.4M & 9866   & http://www.sentiment140.com\\ 
     FlightDelay & 7.3M & 630  & https://catalog.data.gov/dataset/airline-on-time-performance-and-causes-of-flight-delays-on-time-data \\ 
     NYCTaxi & 10M & 21 & http://www.nyc.gov/html/tlc/html/about/trip\_record\_data.shtml \\ 
     HEPMASS & 10M & 28 & https://archive.ics.uci.edu/ml/datasets/HEPMASS \\ 
     HIGGS & 10.6M & 28 & https://archive.ics.uci.edu/ml/datasets/Higgs \\ \hline
    \end{tabular}  
     {\caption{Dataset Description. $|\dd|$ and $|\ff|$ are the number of records and features respectively\label{table:dataset}}} 
  \end{center}
\end{table*}

Furthermore, the optimal sample size to minimize the running time could vary for different configurations. If we magically knew that we should use 2M training samples for LightGBM and 16K training samples for all the other configurations, we could save more time and still identify the correct best configuration. 
Unfortunately, the optimal sample size for each configuration is unknown. A natural idea is to increase the sample size gradually, until a plateau is reached in the learning curve.
However, a naive plateau estimator based on the learning curve is error-prone. As shown from Figure~\ref{fig:obs}, LightGBM's learning curve is flat from 32K to 128K. If we stopped increasing the sample size for it, it would be mis-pruned. Therefore, a more robust strategy is needed.

\begin{figure}[t]
\centering
\includegraphics[scale=0.285]{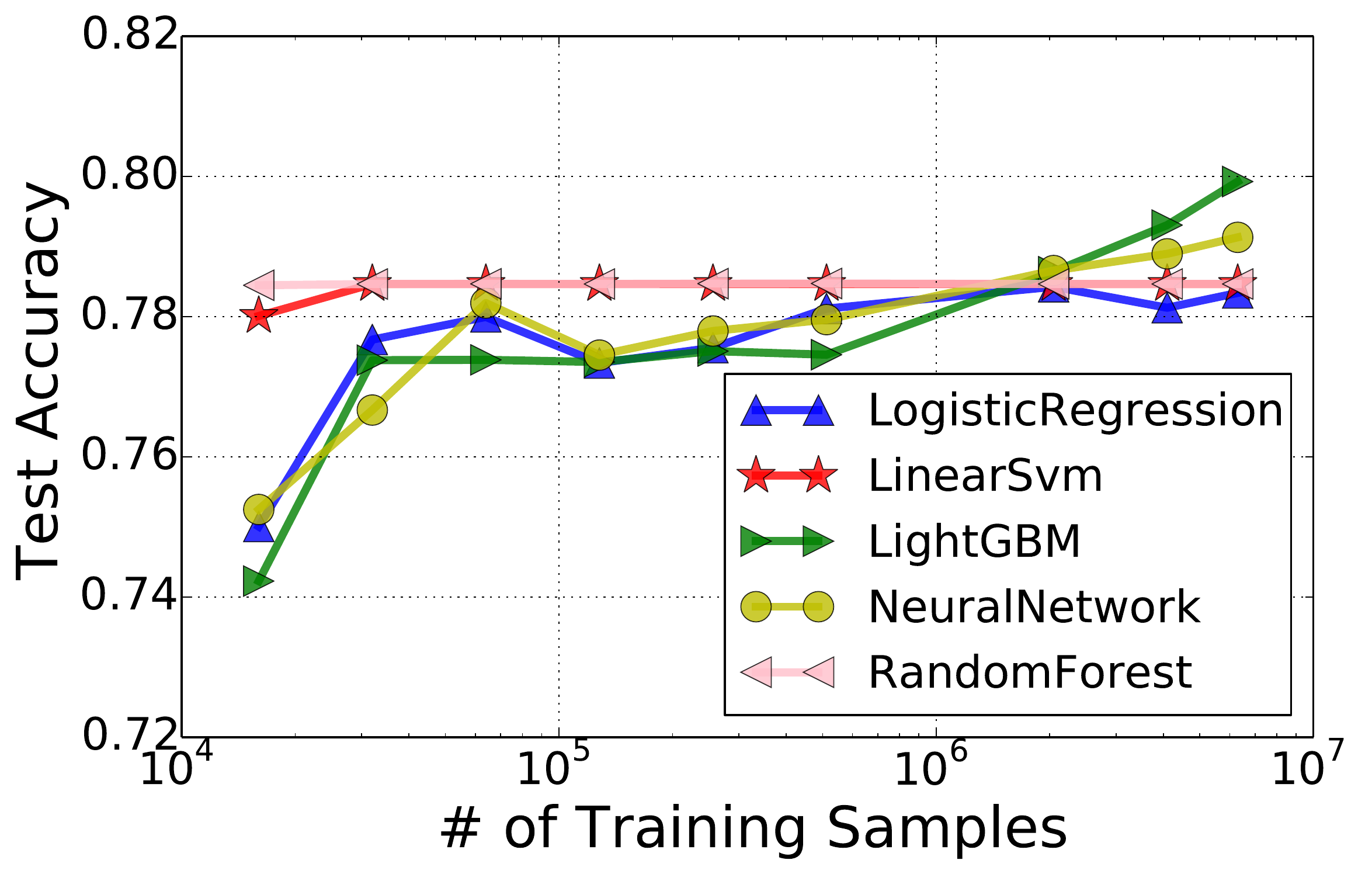}
\caption{Learning Curve\label{fig:obs}}
\end{figure}

\stitle{Overview.} The main idea is to estimate the {\em confidence interval} (CI) of each configuration's real test accuracy
with sampled data, instead of simply using a point estimation of the real test accuracy. In each round, we train the classifier for a selected configuration on some sampled training data. We call such a round of training a {\em probe}. After a probe, we update the confidence interval for the configuration. 
As the sample size increases, the confidence interval shrinks, and the badly-performing configurations can be pruned based on the CIs. 
The pruning based on CI is more robust than based on random observations from the learning curve.

\begin{algorithm}[h]
  \SetAlgoLined\SetAlgoNoEnd
  {\bf Input:} configuration set $\cc$, accuracy loss threshold $\epsilon$\;
  {\bf Output:} the approximate best configuration\;  
   Initialization: $C_{prob} \leftarrow C_1$, $C_{i'}\leftarrow C_1$, $\Omega \leftarrow \cc$\;
   \While{$|\Omega|>1$}{ 
    {\sc Probe}($C_{prob}$)\;    
    $[C_{prob}.l,C_{prob}.u]\gets$\plateau($C_{prob}$) \;
    \lIf{$C_{prob}.l > C_{i'}.l$}{
    $C_{i'} \leftarrow C_{prob}$}   
    \For(\tcp*[f]{pruning}){$C \in \Omega$}{
      \lIf{$C.u - C_{i'}.l\leq \epsilon$}{
      $\Omega\gets \Omega - C$} 
    }     
    \If{Pruning happens}{
    	\For{$C \in \Omega$} {
    		$C.u_{old} \leftarrow C.u$;
    		$C.l_{old} \leftarrow C.l$;
    	}
    }
    $C_{prob} \leftarrow$ \scheduler($\Omega$)     
    }
    \Return $C_{i'}$\;
   \caption{\aml}
   \label{alg:CI-based}
  \end{algorithm}

\stitle{Detailed Algorithm.} \aml proceeds round by round as shown in Algorithm~\ref{alg:CI-based}, where each configuration $C_i$ is annotated with its current sample size ($C_i.s$), {\new current lower bound ($C_i.l$), current upper bound ($C_i.u$), and the cached lower bound ($C_i.l_{old}$) and upper bound ($C_i.u_{old}$) in the recent pruning round.}
In each round within the while loop (line 4), it first probes the configuration $C_{prob}$ (line 5). Then it calls a \plateau subroutine to quickly estimate the confidence interval for $\calA_{prob}$ (line 6). Next, it prunes badly-performing configurations (line 7-9). 
Line 7 identifies the configuration $C_{i'}$ with the largest lower bound. Line 8-9 prunes an configuration if its upper bound is within $\epsilon$ away from the largest lower bound.
{\new If any configuration is pruned (line 10), we call this iteration a {\em snapshot} and will update $C.l_{old}$ and $C.u_{old}$ for each configuration $C$ in this snapshot (line 11-12).} 
At last, it calls a \scheduler subroutine to determine which configuration to probe next as well as its sample size (line 10).

We describe \plateau and \scheduler in the next two sections. 

\section{CI Estimator} \label{sec:plateau}
In this section, we will derive a \plateau for each configuration's real test accuracy, based on the probe over sampled data. For configuration $C_i$, the confidence interval $[l_i, u_i]$ needs to contain the real test accuracy $\calA_i$ with high probability. The computation of $l_i$ and $u_i$ needs to be efficient, i.e., no slower than the probe. In the following, we assume $i$ is fixed and omit it in the notations.

At the first glance, the CI estimation may remind readers of the generalization error bounds (e.g., VC-bound).
The generalization error bound is a universal bound of the difference between each hypothesis's accuracy in training data and its accuracy in infinite data following the same distribution. Nevertheless, the confidence interval we need is the range of the real test accuracy of the hypothesis $H_{tr}$ trained from full training data, while we only have the hypothesis $H_{S_{tr}}$ trained from a sample $S_{tr}\subset D_{tr}$. Therefore, we cannot apply generalization error bound to obtain our confidence interval. 

We use Figure~\ref{fig:plateauEst} to summarize the notations and their relationships which are important for understanding the theoretical results.
$H_{tr}$, $H_{S_{tr}}$, and $H^*$ correspond to the returned hypothesis after training a fixed configuration with full training dataset $\dd_{tr}$, the sampled training dataset $S_{tr}$, and the full dataset $\dd$ respectively.  
For instance, Figure~\ref{fig:plateauEst}(a) shows \del{Next, let us review }the overall derivation relationships among  $\dd$, $\dd_{tr}$, $\dd_{te}$, $S_{tr}$, $S_{te}$, and $H_{S_{tr}}$. First, the full training data $\dd_{tr}$ and the full testing data $\dd_{te}$ are randomly split from the whole data $\dd$. Second, the sampled training data $S_{tr}$ and the sampled testing data $S_{te}$ are randomly drawn from the full training data $\dd_{tr}$ and full testing data $\dd_{te}$. Last, $H_{S_{tr}}$ is trained from the sampled training data $S_{tr}$. 
Note that the CI estimator only has access to $H_{S_{tr}}, S_{tr}$ and $S_{te}$. Though $H_{tr}$ and $H^*$ are not accessible, they are useful in our analysis.

\stitle{Upper bound.} The intuition behind the confidence interval estimation is that we need to relate the two hypotheses $H_{S_{tr}}$ and $H_{tr}$, and use the information we have on $H_{S_{tr}}$ to infer the performance of $H_{tr}$. To upper bound the accuracy of $H_{tr}$, we leverage a \emph{fitness} condition: The training process produces a hypothesis that fits the training data. When the configuration is fixed, the accuracy in a dataset $D$ of the hypothesis trained on $D$ should be no lower than the hypothesis trained on a different dataset $D'\neq D$. 
It is the only assumption we need to prove the upper bound, no matter what training algorithm is used. Under this condition, we found an inequality chain to connect the training accuracy $\calA(H_{S_{tr}},S_{tr})$ to the real test accuracy of $H_{tr}$.

\begin{figure}[t!]
\centering
\includegraphics[scale=0.75]{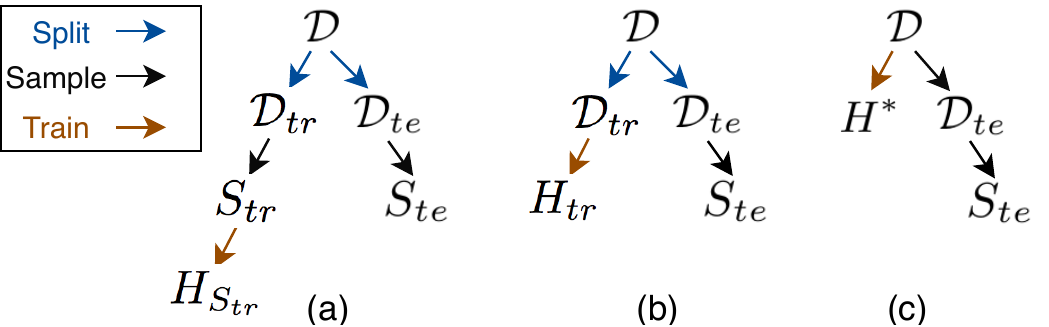}
\caption{Notations Used in CI Estimation and Analysis } 
\label{fig:plateauEst}
\end{figure}

\begin{theorem}[Upper Bound]\label{thn:upper}
    Under the fitness condition, with probability at least $1-\frac{\delta}{2n^2}$, $\calA(H_{tr},\dd_{te}) \leq u$, where
    \begin{align*}
    u&\triangleq\calA(H_{S_{tr}},S_{tr}) + (\frac{1}{2|S_{tr}|}\ln \frac{4n^2}{{\delta}})^{\frac{1}{2}} + (\frac{1}{2|\dd_{te}|}\ln \frac{4n^2}{{\delta}})^{\frac{1}{2}}
    \end{align*}    
\end{theorem}

{
\begin{proof}
Let us first recall the fitness condition.  Given a fixed configuration $C_i$, let $H$ and $H'$ be the hypothesis returned by training on two different sample sets $D$ and $D'$, respectively. Note that $H$ and $H'$ are both from the same fixed hypothesis space $\hh_i$. Our assumption is that {\em $H$ has no lower accuracy on $D$ than $H'$}. Symmetrically, {\em $H'$ has no lower accuracy on $D'$ than $H$}. 

First, let us break down {$\calA(H_{tr},\dd_{te}) - \calA(H_{S_{tr}},S_{tr})$} into four clauses, as shown in Equation~\eqref{eq:est}. 

\begin{equation} \label{eq:est}
\begin{split}
&\calA(H_{tr},\dd_{te}) - \calA(H_{S_{tr}},S_{tr})\\ 
 =  & [\calA(H_{tr},\dd_{te}) - \calA(H^*,\dd_{te})] + [\calA(H^*,\dd_{te}) - \calA(H^{*},\dd)]\\
 + & [\calA(H^{*},\dd) - \calA(H^{*},S_{tr})] + [\calA(H^{*},S_{tr}) - \calA(H_{S_{tr}},S_{tr})]
\end{split}
\end{equation}

Since $\dd_{tr}$ and $\dd_{te}$ are randomly split from $\dd$, we have $\dd = \dd_{tr} \cup \dd_{te}$. Let $x = \frac{|\dd_{te}|}{|\dd|}$ be the hold-out ratio. For any $H\in \hh$, we have:

\begin{equation} \label{eq:accumulate}
\begin{split}
x\calA(H,\dd_{te}) + (1-x)\calA(H,\dd_{tr}) = \calA(H,\dd) \\
\Rightarrow \calA(H,\dd_{te}) = \frac{1}{x}[\calA(H,\dd) - (1-x)\calA(H,\dd_{tr})]
\end{split}
\end{equation}

Next, apply Equation~\eqref{eq:accumulate} to the first clause in Equation~\eqref{eq:est}:

\begin{equation} \label{eq:clause1}
\begin{split}
& \calA(H_{tr},\dd_{te}) - \calA(H^*,\dd_{te}) \\
= & \frac{1}{x}[\calA(H_{tr}, \dd) - \calA(H^*, \dd) ]\\
- & \frac{1-x}{x}[\calA(H_{tr}, \dd_{tr})- \calA(H^{*}, \dd_{tr})] \leq 0  
\end{split}
\end{equation}
The inequality is derived from the fitness assumption (recall from Figure~\ref{fig:plateauEst} that $H_{tr}$ is trained from $\dd_{tr}$ and $H^*$ is trained from $\dd$).

Next, we bound the second clause in Equation~\eqref{eq:est} with Hoeffding inequality: With probability at least $(1-\frac{\delta}{4n^2})$,
\begin{equation} \label{eq:clause2}
\begin{split}
& \calA(H^*,\dd_{te}) - \calA(H^{*},\dd) \leq (\frac{1}{2|\dd_{te}|}\ln \frac{4n^2}{{\delta}})^{\frac{1}{2}} \\
\end{split}
\end{equation}

Similarly, with probability at least  $(1-\frac{\delta}{4n^2})$,
\begin{equation} \label{eq:clause3}
\begin{split}
& \calA(H^*,\dd) - \calA(H^{*},S_{tr}) \leq (\frac{1}{2|S_{tr}|}\ln \frac{4n^2}{{\delta}})^{\frac{1}{2}} \\
\end{split}
\end{equation}

Please note that Equation~\eqref{eq:clause2} and \eqref{eq:clause3} will not hold if we replace $H^*$ with $H_{S_{tr}}$. 
That is because for hypothesis $H_{S_{tr}}$, $\dd_{te}$ and $S_{tr}$ cannot be regarded as random samples, since $H_{S_{tr}}$ is tailored to the sample set $S_{tr}$. Therefore, introducing $H^*$ is necessary in our analysis.

Last, since $H_{S_{tr}}$ is trained from $S_{tr}$, by the fitness assumption we have:
\begin{equation} \label{eq:clause4}
\begin{split}
\calA(H^{*},S_{tr}) - \calA(H_{S_{tr}},S_{tr}) \leq 0 
\end{split}
\end{equation}

By substituting the four clauses in Equation~\eqref{eq:est} with Equation~\eqref{eq:clause1}-\eqref{eq:clause4}, we obtain Theorem~\ref{thn:upper} using union bound.
\end{proof}
}

Note that the computation of $\calA(H_{S_{tr}},S_{tr})$ is no slower than the probing (i.e., training with sampled data). In fact, the evaluation is usually much more efficient than training for the same scale of dataset. 

\stitle{Lower Bound.} The lower bound is easier due to the {\emph{exploitativeness}} presumption discussed in the problem formulation: Full training data produce better hypothesis than sampled training data for a fixed configuration. 
The real test accuracy of $H_{tr}$ can then be lower bounded by $\calA(H_{S_{tr}},D_{te})$.
However, the computation of $\calA(H_{S_{tr}},D_{te})$ can be slower than probing, if $|D_{te}|\gg|S_{tr}|$. To make the CI estimation efficient, we also sample the testing data. We denote the sampled testing data as $S_{te}$. We can then lower bound $\calA(H_{S_{tr}},D_{te})$ by $\calA(H_{S_{tr}},S_{te})$ minus a variation term. 
\begin{theorem}[Lower Bound]\label{thn:lower} 
	Under the exploitativeness assumption, with probability at least $1- \frac{\delta}{2n^2}$, 
 \[\calA(H_{{tr}},\dd_{te}) \geq l\triangleq \calA(H_{S_{tr}},S_{te}) - (\frac{1}{2|S_{te}|}\ln {\frac{2n^2}{\delta}})^{\frac{1}{2}}\] 
\end{theorem}
{
\begin{proof}
First, the exploitativeness assumption states that
\begin{equation} \label{eq:fullBest}
\begin{split}
\calA(H_{tr},\dd_{te}) \geq \calA(H_{S_{tr}},\dd_{te})
\end{split}
\end{equation}
Next, 
based on Hoeffding inequality, with probability at least $1- \frac{\delta}{2n^2}$,
\begin{equation} \label{eq:testAcc}
\begin{split}
    \calA(H_{S_{tr}},S_{te}) - \calA(H_{S_{tr}},\dd_{te}) \leq (\frac{1}{2|S_{te}|}\ln {\frac{2n^2}{\delta}})^{\frac{1}{2}}
\end{split}
\end{equation}
Combining Equation~\eqref{eq:fullBest} and \eqref{eq:testAcc}, we have
{$\calA(H_{tr},\dd_{te}) \geq l $}
with probability at least $1- \frac{\delta}{2n^2}$.
\end{proof}  
} 

{\new 
With Theorem~\ref{thn:upper} and \ref{thn:lower}, we can now estimate the current lower bound and upper bound of the probing configuration $C_{prob}$. As shown in Algorithm~\ref{alg:plateau}, we first initialize the current lower bound $C_{prob}.l$ and upper bound $C_{prob}.u$ according to Theorem~\ref{thn:upper} and \ref{thn:lower}. Furthermore, we add a constraint that the current CI must be contained in the CI of the last snapshot where pruning happens, i.e., $[C_{prob}.l, C_{prob}.u] \subset [C_{prob}.l_{old}, C_{prob}.u_{old}]$. Thus, if $C_{prob}.l < C_{prob}.l_{old}$, we replace $C_{prob}.l$ with $C_{prob}.l_{old}$ (line 4). Similar for $C_{prob}.u$ (line 5). In this way, we can guarantee that the CI for each configuration shrinks from one snapshot to another snapshot where pruning happens. Recall that $C_{prob}.l_{old}$ and $C_{prob}.u_{old}$ get updated in each snapshot.
}


\begin{algorithm}[h]
    \SetAlgoLined
    {\bf Input:} $C_{prob}, \calA_{prob}(H_{S_{tr}},S_{tr}), \calA_{prob}(H_{S_{tr}},S_{te})$ \;
    {\bf Output:} $[C_{prob}.l, C_{prob}.u]$\;
    $[C_{prob}.l, C_{prob}.u] \leftarrow$ Theorem~\ref{thn:upper} and \ref{thn:lower} \;
    \lIf{$C_{prob}.l < C_{prob}.l_{old}$} {$C_{prob}.l \leftarrow C_{prob}.l_{old}$}
    \lIf{$C_{prob}.u > C_{prob}.u_{old}$} {$C_{prob}.u \leftarrow C_{prob}.u_{old}$}
    \Return $[C_{prob}.l, C_{prob}.u]$ \;
     \caption{\plateau}
     \label{alg:plateau}
\end{algorithm}

\subsection{Correctness of Algorithm~\ref{alg:CI-based} \label{sec:correct}}
Theorem~\ref{thn:prob1} shows that Algorithm~\ref{alg:CI-based} can successfully return an approximate best configuration with high probability.
\begin{corollary}[Confidence Interval]\label{thn:CI}
With probability at least $1- \frac{\delta}{n^2}$, 
 { $\calA_i \in [l_i,u_i]$}.
\end{corollary}

\begin{theorem}[Correctness]\label{thn:prob1}
With probability at least $1-\delta$, 
Algorithm~\ref{alg:CI-based} returns the approximate best configuration $C_{i'}$ with $\calA_{i^*} - \calA_{i'} \leq \epsilon$. 
\end{theorem}
\begin{proof}
Without loss of generality, we assume $C_1$ is the returned configuration $C_{i'}$ by Algorithm~\ref{alg:CI-based}. We denote the number of snapshots as $R$. We will prove that when all the confidence intervals at all the snapshots correctly bound the real test accuracy (denoted as event $E$), the algorithm returns correct approximate configuration.\del{In this event, we only need to consider the case when the confidence intervals do not expand in the iterations in $R$ as the sample size increases, because otherwise, we can force the confidence intervals not to expand, while still correctly bounding the real test accuracy.} We show that when $E$ happens, for each pruned configuration $2\leq i\leq n$, 
$\calA_i \leq \calA_1 + \epsilon$.  

Consider snapshot $r\in [R]$ and let $l_{i_r}^{(r)}$ be the highest lower bound in this snapshot, and $C_{p}$ be a pruned configuration in this snapshot. 
When $E$ happens, $\calA_{p} \leq u_{p}$. And $u_{p} \leq l_{i_r}^{(r)} + \epsilon$ according to line 9 in Algorithm~\ref{alg:CI-based}. So the pruned configuration must satisfy $\calA_{p} \leq l_{i_r}^{(r)} + \epsilon$. 
Furthermore, as Algorithm~\ref{alg:CI-based} proceeds to snapshot $r+1$, $l_{i_r}^{(r)} \leq l_{i_{r+1}}^{(r+1)}${\new. This is because the lower bound of configuration $C_{i_r}$ does not decrease from snapshot $r$ to snapshot $r+1$ according to Algorithm~\ref{alg:plateau}, i.e., $l_{i_r}^{(r)} \leq l_{i_r}^{(r+1)}$. 
In addition, $l_{i_r}^{(r+1)} \le l_{i_{r+1}}^{(r+1)}$ in snapshot $r+1$,} according to line 7 in Algorithm~\ref{alg:CI-based}.
Thus, $\calA_{p} \leq l_{i_R}^{(R)}+\epsilon = l_{1} + \epsilon$
by induction on the snapshot number $r$. 
\del{Recall that $C_1$ is the configuration with the highest lower bound in the last iteration and is returned by Algorithm~\ref{alg:CI-based}. }
Also, since $l_1 + \epsilon \leq \calA_1 + \epsilon$, we have $\calA_{p_r} \leq \calA_{1} + \epsilon$.

Next, {\new let $E_r$ be the event that all the confidence intervals at iteration $r$ correctly bound the real test accuracy where $r\in [R]$. We can decompose event $\bar{E}$ as $\bar{E}=\cup_{r=1}^R\bar{E_r}$, where $\bar{E}$ (resp. $\bar{E_r}$) is the opposite event of $E$ (resp. $E_r$).} According to Corollary~\ref{thn:CI}, the derived confidence interval $[l_i^{(r)}, u_i^{(r)}]$ at snapshot $r$ is correct with probability at least $1-\frac{\delta}{n^2}$ for any configuration $C_i$. 
{\new Hence, the probability of $\bar{E}_{r}$ is at most $\frac{\delta}{n}$ according to Algorithm~\ref{alg:plateau} and the union bound. Furthermore, we have $n-1$ pruned configurations across all iterations, so $R<n$. Consequently, the probability of $\bar{E}$ is below $\delta$ by applying union bound to the sub-events $\bar{E}_{r},r\in[R]$. That is, event $E$ happens with probability at least $1-\delta$.}
\del{Hence, for each round $r\in R$, the probability of having the correct confidence interval for each configuration is at least $1-\frac{\delta}{n}$ according to the union bound.
In total, we have $n-1$ pruned configurations across all iterations, so $R<n$.
Based on the union bound over $R$, we know that event $E$ happens with probability at least $1-\delta$.}
Thus, we have $\calA_{i^*} - \calA_{i'} \leq \epsilon$ with probability at least $1-\delta$.
\end{proof}

\stitle{Discussion.} 
Our upper confidence bound $u$ has an additive form with three components: the training accuracy on the sampled training dataset $S_{tr}$, a variation term due to training sample size $|S_{tr}|$, and a variation term due to full testing data size $|\dd_{te}|$. Intuitively, $u$ increases as the training accuracy $\calA(H_{S_{tr}},S_{tr})$ increases, because higher training accuracy indicates higher potential of the configuration's learning ability. But that potential decreases as the training sample size $|S_{tr}|$ increases, because the more data we have used, the less room for improvement by adding more training data. Finally, since the real test accuracy is measured in the full testing data $\dd_{te}$, the variation due to the random split needs to be added to $u$. The larger $\dd_{te}$ is, the smaller the variation is. Both the variation terms are affected by the confidence probability $1-\frac{\delta}{2n^2}$. Higher confidence probability corresponds to wider confidence interval, thus larger $u$. In sum, $u$ is positively correlated with the training accuracy and the number of configurations $n$, and negatively correlated with the training sample size and full testing data size.

Our lower confidence bound $l$ is expressed as the accuracy of $H_{S_{tr}}$ in the sampled testing dataset $S_{te}$, minus a variation term due to testing sample size $|S_{te}|$. 
As $|S_{te}|$ increases, the difference between $\calA(H_{S_{tr}},S_{te})$ and $\calA(H_{S_{tr}},\dd_{te})$ becomes smaller, and the lower bound rises. 
Higher confidence probability $1- \frac{\delta}{2n^2}$ corresponds to smaller $l$. 
In sum, $l$ is positively correlated with the testing sample size and the testing accuracy in the sample, and negatively correlated with $n$.

We have used the exploitativeness assumption in deriving the lower bound: {\small $\calA_i(H_{S_{tr}}, D_{te}) \leq \calA_i(H_{tr}, D_{te})$} for any $C_i$. We argue that even though this assumption is not exactly satisfied in practice, it holds closely enough to provide useful results. That is, in most cases this assumption holds, and when this assumption is violated, we can perform a post-processing step after the algorithm finishes. If there exists {\small $\calA_{i'}(H_{S_{tr}}, D_{te}) > \calA_{i'}(H_{tr}, D_{te})$} for the selected configuration $C_{i'}$, the user could use {\small $H_{S_{tr}}$} instead of {\small $H_{tr}$} as the final classifier. 
First, that satisfies users’ preference in finding a more accurate classifier. Second, it holds the $\epsilon$-guarantee, because the lower bound $l_{i'}$ holds for {\small $H_{S_{tr}}$} of configuration $C_{i'}$, and it is no lower than the pruning lower bounds $l_{i_r}, \forall r\in [R]$. 

The correctness of our algorithm is independent of the choice of the scheduler.

\section{Scheduler}\label{sec:scheduler}
This section focuses on the optimization part in Problem~\ref{prob:bestIdentify}, i.e., how to minimize the total running time. 
Let $T_i(s)$ be the probing time with a sampled training dataset size $s$ for configuration $C_i$, and  $t_i$ be the accumulated running time for probing configuration $C_i$ in Algorithm~\ref{alg:CI-based}. Also, let $l_i$ and $u_i$ be the lower bound and upper bound respectively for configuration $C_i$ when the algorithm terminates. 
With these notations, the design of \scheduler in \aml can be expressed as a constrained optimization problem. Without loss of generality, assume $C_1$ is returned by Algorithm~\ref{alg:CI-based}.

\begin{problem}[Scheduling]\label{prob:sch}
Design a scheduler to minimize $\calT = \sum_i {t}_i$\del{ + $T_1(|\dd_{tr}|)$}, subject to:
\begin{equation*} 
\begin{split}
& u_2 \leq l_1 + \epsilon, u_3 \leq l_1 + \epsilon, \dots, u_n \leq l_1 + \epsilon \\
\end{split}
\end{equation*}

\end{problem}

The objective function in Problem~\ref{prob:sch} is the time taken to select the approximate best configuration. Since probing dominates the running time in each iteration, we use the total time of all probes as the proxy of the selection time. The constraints in Problem~\ref{prob:sch} ensure that all the configurations except $C_1$ are pruned. They are necessary for the termination of Algorithm~\ref{alg:CI-based}. 

To solve Problem~\ref{prob:sch}, we begin with studying the properties of the `oracle' optimal scheduling scheme when it has access to $t_i$ as a function of $l_i$ and $u_i$ respectively, i.e., $t_i = f_i(l_i)$ and $t_i = g_i(u_i)$, after the samples are drawn. 
We claim that the optimal scheduling scheme with this oracle access probes each configuration uniquely once, since otherwise we can always reduce the total running time by only keeping the last probe. 
Our objective function can be rewritten as $f_1(l_1) + g_2(u_2) + \cdots + g_n(u_n)$. Furthermore, by applying the method of Lagrange multipliers, 
we obtain the conditions the optimal solution must satisfy:
\begin{equation}\label{eqn:constraints} 
\left\{
\begin{array}{lll}
\frac{df_1}{dl_1} = -(\frac{dg_2}{du_2}+\cdots+\frac{dg_n}{du_n})\\
\\
l_1 + \epsilon = u_2 = \cdots = u_n
\end{array}
\right. 
\end{equation}
Now, since we do not have oracle access to $f_i$ and $g_i$, there is no closed-form formula to decide the optimal sample size $s_i^*$ for configuration $C_i$.  
To solve this challenge, We propose a scheduling scheme \grad with two parts. 

First, we use the gradient of the running time with respect to the confidence interval to
determine the configuration to probe next. We depict this strategy in Algorithm~\ref{alg:gradient}. \grad first sorts the remaining configuration set $\Omega$ in descending order of the upper bound (line 3), and make a guess ($\Omega_1$) on the best configuration $C_1$. 
Next, it compares the gradient $\frac{\Delta T_1}{\Delta l_1}$ with $|\frac{\Delta T_2}{\Delta u_2}+\cdots+\frac{\Delta T_n}{\Delta u_n}|$: If $\frac{\Delta T_1}{\Delta l_1}$ is smaller, then configuration $\Omega_1$ with the largest upper bound is picked for the next probe (line 4); otherwise, configuration $\Omega_2$ with the second largest upper bound is picked (line 5). Here, $\Delta T_i$ denotes the running time difference between the recent two consecutive probes on $C_i$, and $\frac{\Delta T_i}{\Delta l_i}$ serves as the proxy of $\frac{\Delta f_i}{\Delta l_i}$ (similar for $\frac{\Delta g_i}{\Delta u_i}$). The choice between $\Omega_1$ and others is based on the first condition in Equation~\eqref{eqn:constraints}. If the unit cost of increasing the lower bound of $\Omega_1$ is smaller than the sum of the unit cost of decreasing the upper bounds of the other configurations, then we opt to probe $\Omega_1$. The choice of $\Omega_2$ among $\Omega_2$ to $\Omega_n$ is based on the second condition in Equation~\eqref{eqn:constraints}, towards attaining the same upper bound for them.

Second, we design the sample size sequence within each configuration. 
As shown in line 6 of Algorithm~\ref{alg:gradient}, we utilize a common trick called {\em geometric scheduling}, 
which was used in prior work to increase the sample size for a single configuration~\cite{KDD:provost1999}.     
We further derive the closed-form for the optimal step size $c$, when $T_i(s)$ is a power function over the sample size, i.e., $T_i(s) = s^\alpha$ where $\alpha$ is a real number. The optimal step size follows $c = 2^{\frac{1}{\alpha}}$. 
Details can be found in the Appendix.

\begin{algorithm}[h]
    \SetAlgoLined
    {\bf Input:} Remaining configurations $\Omega$\;
    {\bf Output:} Configuration for next probe $C_{prob}$\;
    sort\_by\_upper\_bound($\Omega$)\; 
    \lIf{$\frac{\Delta T_1}{\Delta l_1} \leq |\frac{\Delta T_2}{\Delta u_2}+\cdots+\frac{\Delta T_{n}}{\Delta u_n}|$}{
    $C_{prob} \leftarrow \Omega_1$
    }\lElse{
    $C_{prob} \leftarrow \Omega_2$
    }
    $C_{prob}.s \leftarrow c \times C_{prob}.s$\; 
    \Return $C_{prob}$ \;
     \caption{\scheduler--\grad}
     \label{alg:gradient}
    \end{algorithm}

\stitle{Performance Analysis.} {In practice, \aml is used in two scenarios. Scenario (i): during exploration, users want to try a few configurations (e.g., verifying usefulness of a few new features) as an intermediate step. The best configuration will decide the follow-up trials, but it does not serve as the final configuration, and does not require full training. Scenario (ii): at the end of the exploration, users need to get the trained classifier corresponding to the selected configuration $C_{i'}$. 
In scenario (ii), the total running time involves not only the time to select the approximate best configuration, but also the time taken to train the classifier on full data with $C_{i'}$. When $C_{i'}$ is fixed, both scenario (i) and (ii) share the same optimal scheduler.
Under certain conditions, we are able to prove a 4-approx guarantee for \grad. Please find details in the Appendix.}

{\new
\stitle{Remark.} We have also analyzed a well-known scheduling scheme, called \ucb. Theoretical and experimental results can be found in the Appendix.
}
\section{Experiments}\label{sec:exp}
This section evaluates the efficiency and effectiveness of our \aml module. 
First, we evaluate whether \aml successfully selects top configuration and meanwhile reduces the total running time. 
Second, we compare the CI-based pruning with an existing pruning algorithm based on point estimation.
We also compare different scheduling schemes in the Appendix.

 \begin{figure*}[t]

    \begin{subfigure}{0.69\textwidth}
        \centering
        \includegraphics[width=\linewidth]{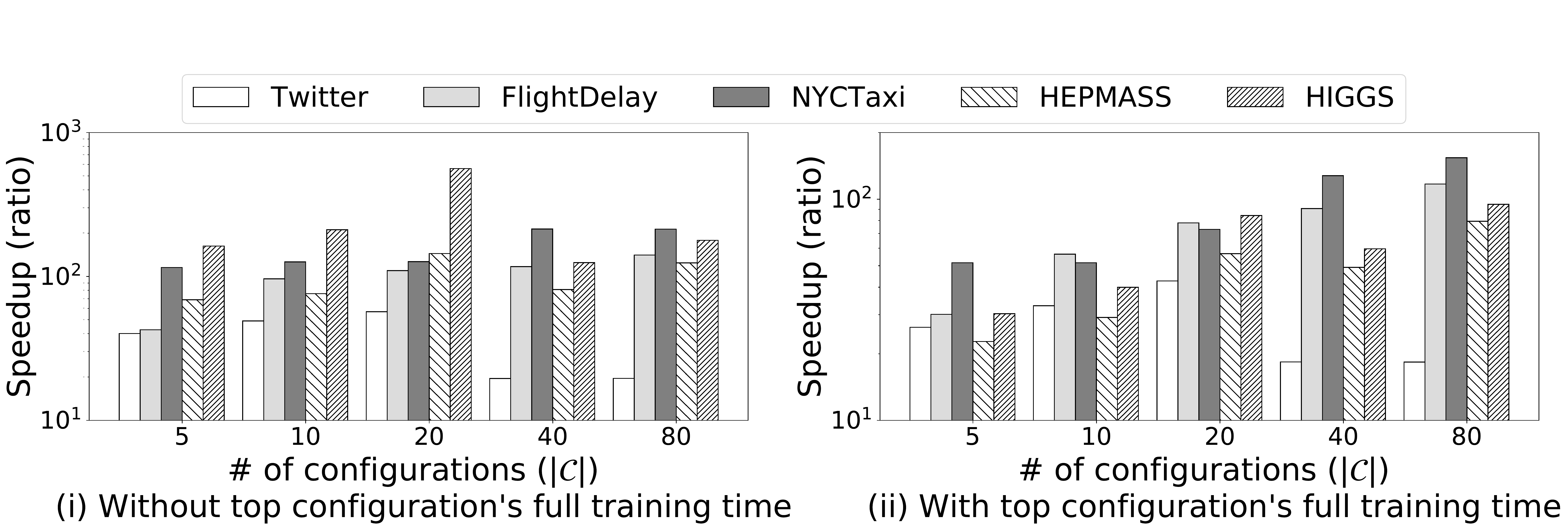}
        \caption{Speedup Comparison}
        \label{fig:speedup}
    \end{subfigure}
    \begin{subfigure}{.3\textwidth}
        \centering
        \includegraphics[width=\linewidth]{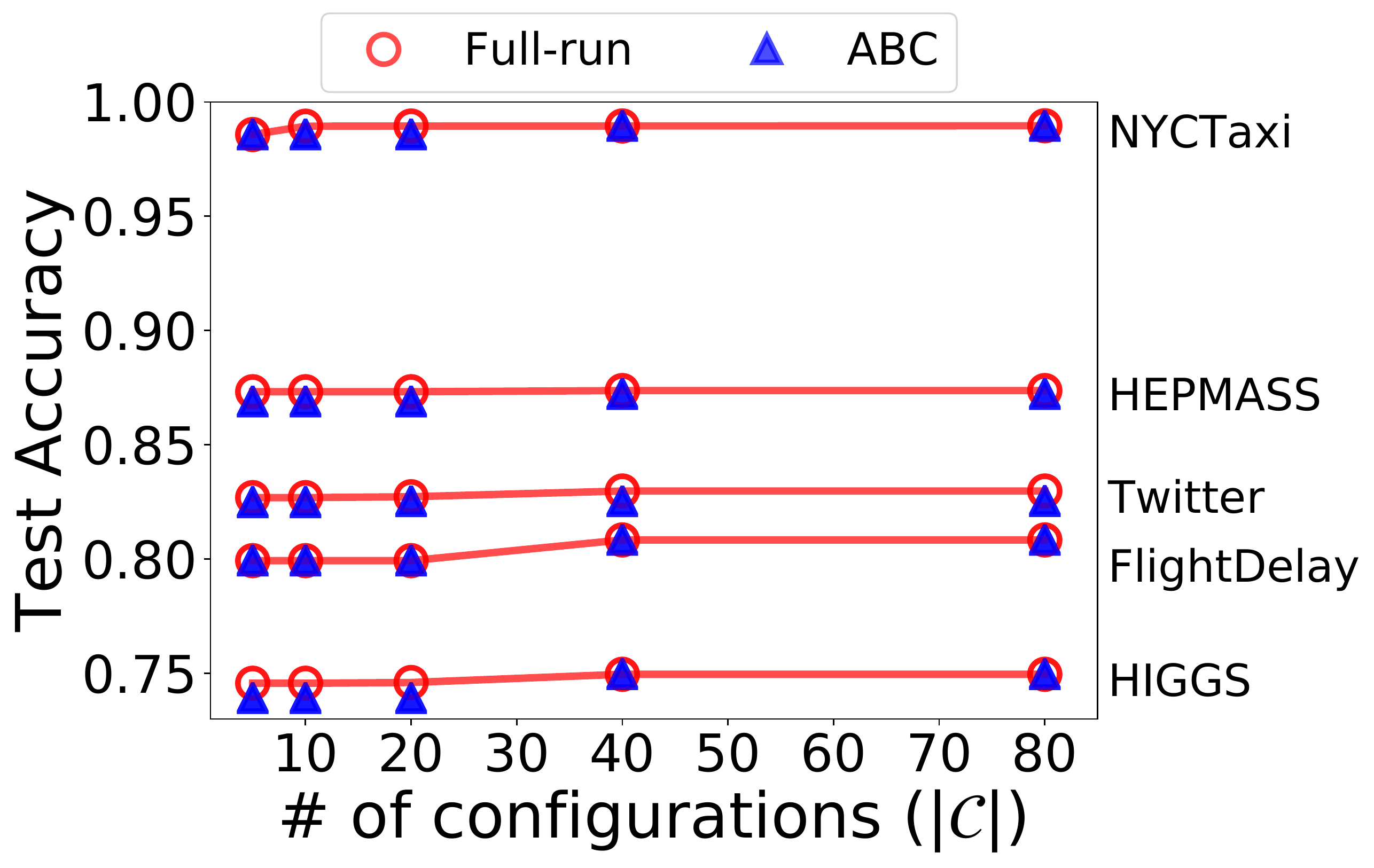}
        \caption{Accuracy Comparison}
        \label{fig:CI_accuracy}
    \end{subfigure}
    \caption{Comparison Between Full-run and \aml}
    \label{fig:comp}
\end{figure*}
     
\subsection{Experimental Setup}\label{ssec:setup}
\stitle{Configurations.} We focus on the task of classifying featurized data in our evaluation. Specifically, we choose five widely used and high-performance learners: {\em LogisticRegression, LinearSVM, LightGBM, NeuralNetwork,} and {\em RandomForest}. Each classifier is associated with various hyperparameters, e.g., the number of trees in RandomForest and the penalty coefficient in LinearSVM. In total they have 29 discrete or continuous hyperparameters. 
In our experiments, we use random search to generate each hyperparameter value from its corresponding domain. 

\stitle{Datasets.} We evaluate with five large-scale machine learning benchmarks that are publicly available. As discussed in introduction, the motivation of \aml is to handle large datasets and quickly select the approximate best configuration. Thus, the datasets evaluated in our experiments are all at the scale of millions of records ($|\dd|$) and with up to 10K features ($|\ff|$). We do not use the AutoML benchmarks such as HPOlib~\cite{NIPSWorkshop:EggFeuBerSnoHooHutLey13hpolib} or OpenML~\cite{KDDExploration:OpenML2013}, which mainly contain small or median-sized datasets (up to 50K records).
The statistics of each dataset are depicted in Table~\ref{table:dataset}. We used min-max normalization for all datasets, and n-gram extraction as well as model-based top-K feature selection for Twitter. The processed datasets are available from \url{https://www.microsoft.com/en-us/research/people/chiw/#!downloads}.

\stitle{Algorithms.} We compare our proposed \aml with the standard approach named {\em Full-run}. 
For each configuration, Full-run first trains the classifier with full training data, and then tests it on the full testing data. Afterwards, it returns the configuration with the highest testing accuracy. This method is supported in mature tools like scikit-learn and Azure ML.
Existing approaches to best configuration selection, such as DAUB~\cite{AAAI:Sabharwal16} or Successive-halving~\cite{AISTATS:jamieson2016non}, are heuristics without accuracy guarantee. Our solution and such heuristics are not apple-to-apple comparison, as they cannot ensure $\epsilon$-approximation guarantee on accuracy. {Nevertheless, we conduct a best effort comparison with them.}

\stitle{Setup.} We conducted our evaluation on a VM
with 8 cores and 56 GB RAM. The initial training sample size and testing sample size are 1000 and 2000 respectively. The geometry step size is set to be $c= 2$. $\epsilon = 0.01, \delta=0.5$. Since $\delta$ is under the $\log$ term, the result is not sensitive to $\delta$. We also conduct experiments with varying $\epsilon$, as shown in the Appendix. 

We use the same set of sampled configurations for both Full-run and \aml. We vary the number of input configurations from 5 to 80. Since we focus on large datasets, it already takes a day or half to finish Full-run with 80 configurations for a single dataset. So unlike the case of small datasets, 80-100 is a realistic number because that is how many configurations a user can try with Full-run within a reasonable period of waiting. 

\subsection{\aml vs. Full-run}\label{ssec:exp_CI}

We compare \aml against Full-run from two perspectives, running time and accuracy.
We first compute the {\em speedup} achieved by \aml, where speedup is defined as the ratio between Full-run's total running time and ours. Next, we compare the configuration $C_{i'}$ returned by our \aml with the best configuration $C_{i^*}$ provided by Full-run in terms of real test accuracy.

\stitle{Efficiency Comparison.} 
{As discussed, \aml is used in two scenarios in practice. During exploration, users only need the selected configuration to decide the follow-up trials, but do not require full training with the selected configuration. At the end of the exploration, users need to train the classifier with the selected configuration in the full data.}
Thus, we evaluate the running time speedup in these two scenarios: {\em (i)} we first compare the selection time between our \aml and Full-run as depicted in Figure~\ref{fig:speedup}(i); {\em (ii)} we then compare the total running time including the time to train the final classifier, in Figure~\ref{fig:speedup}(ii). Our solution is on average 190$\times$ faster than Full-run in scenario (i), and is on average 60$\times$ faster than Full-run in scenario (ii). Furthermore, 23 out of 25 experiments (i.e., 5 different datasets times 5 different configuration set size) has at least $|\cc|\times$ speedup in scenario (i), and 22 out of 25 experiments achieve at least $|\cc|\times$ speedup in scenario (ii). This means that the running time of \aml is even faster than fully evaluating one average configuration in most cases, which further means even a perfectly distributed Full-run can't beat the non-distributed \aml. 

The speedup on dataset Twitter is consistently lower than other datasets. This is mainly because Twitter is one order of magnitude smaller than the other datasets. With the same sample size, the sampling ratio is higher than the other datasets, which causes lower speedup.

\stitle{Effectiveness Comparison.} As illustrated in Figure~\ref{fig:CI_accuracy}, \aml successfully selects the configuration whose real test accuracy is within 0.01 from the best configuration's real test accuracy in {all} of our experiments.
In particular, when $|\cc| = $ 40 or 80, \aml successfully identifies the exact best configuration for FlightDelay, NYCTaxi, and HIGGS. The largest deviation is around 0.0068 when $|\cc| = $ 20 for HIGGS. 
 
\stitle{Takeaway.} Compared to Full-run, our proposed \aml can successfully select a competitive or identical best configuration but with much less time.

\subsection{CI-based Pruning vs. Successive-halving}\label{ssec:exp_sh}
Next, we compare our proposed CI-based pruning with Successive-halving. Successive-halving was proposed as a pruning strategy to evaluate iterative training configurations with a resource budget of the total number of iterations of all configurations. We modify it to use the total sample size as the resource budget. In each round, it trains a classifier with the sampled data for each remaining configuration, and then eliminates the half of the low-performing configurations. This pruning is based on point estimation, i.e., they directly use $\calA(H_{S_{tr}},D_{te})$ to approximate $\calA(H_{tr},D_{te})$, in contrast to using confidence interval as \aml. It repeats until there is only one remaining configuration.

Since the two solutions are designed to satisfy different constraints (accuracy loss and resource), they are not directly comparable. We do our best to make a fair comparison. In this section, we run Successive-halving with identical sample size sequence as \aml, to compare the CI-based pruning and the halving strategy based on point estimation. We perform the same set of experiments as in the previous experiment section for Successive-halving.
We introduce a metric, called {\em relative accuracy loss}, to measure the difference between the returned configuration $C_{i'}$ and the best configuration $C_{i^*}$ in terms of the test accuracy:
    $\Delta_{rel} = \frac{|\calA_{i^*} - \calA_{i'}|}{\calA_{i^*}}$.
The smaller $\Delta_{rel}$ is, the better.

\begin{figure}[h]
\centering
    \begin{subfigure}{0.49\textwidth}
        \centering
        \includegraphics[width=\linewidth]{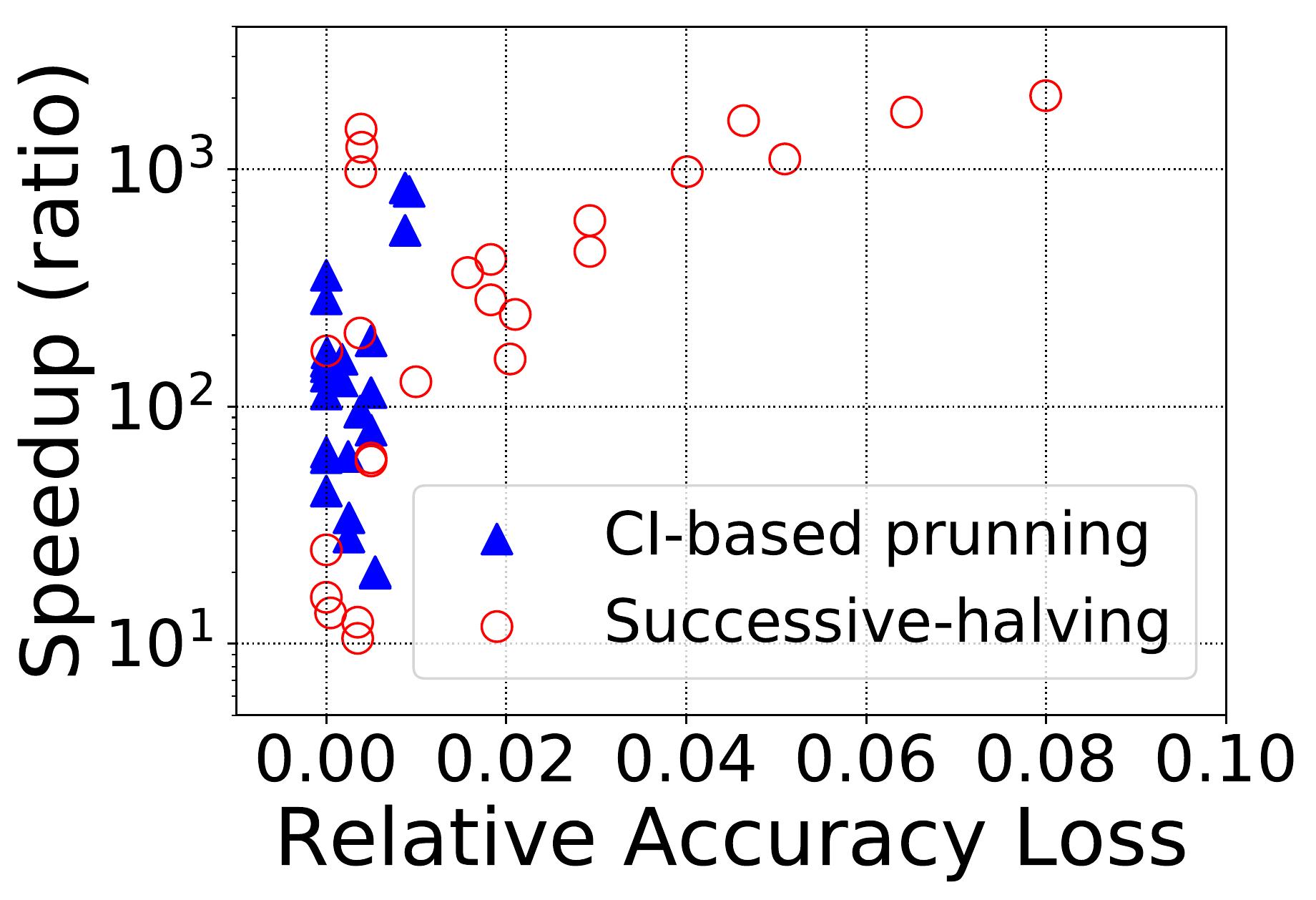}
        \caption{Speedup vs. $\Delta_{rel}$} 
        \label{fig:SH_scatter}
    \end{subfigure}
    \begin{subfigure}{.49\textwidth}
        \centering
        \includegraphics[width=\linewidth]{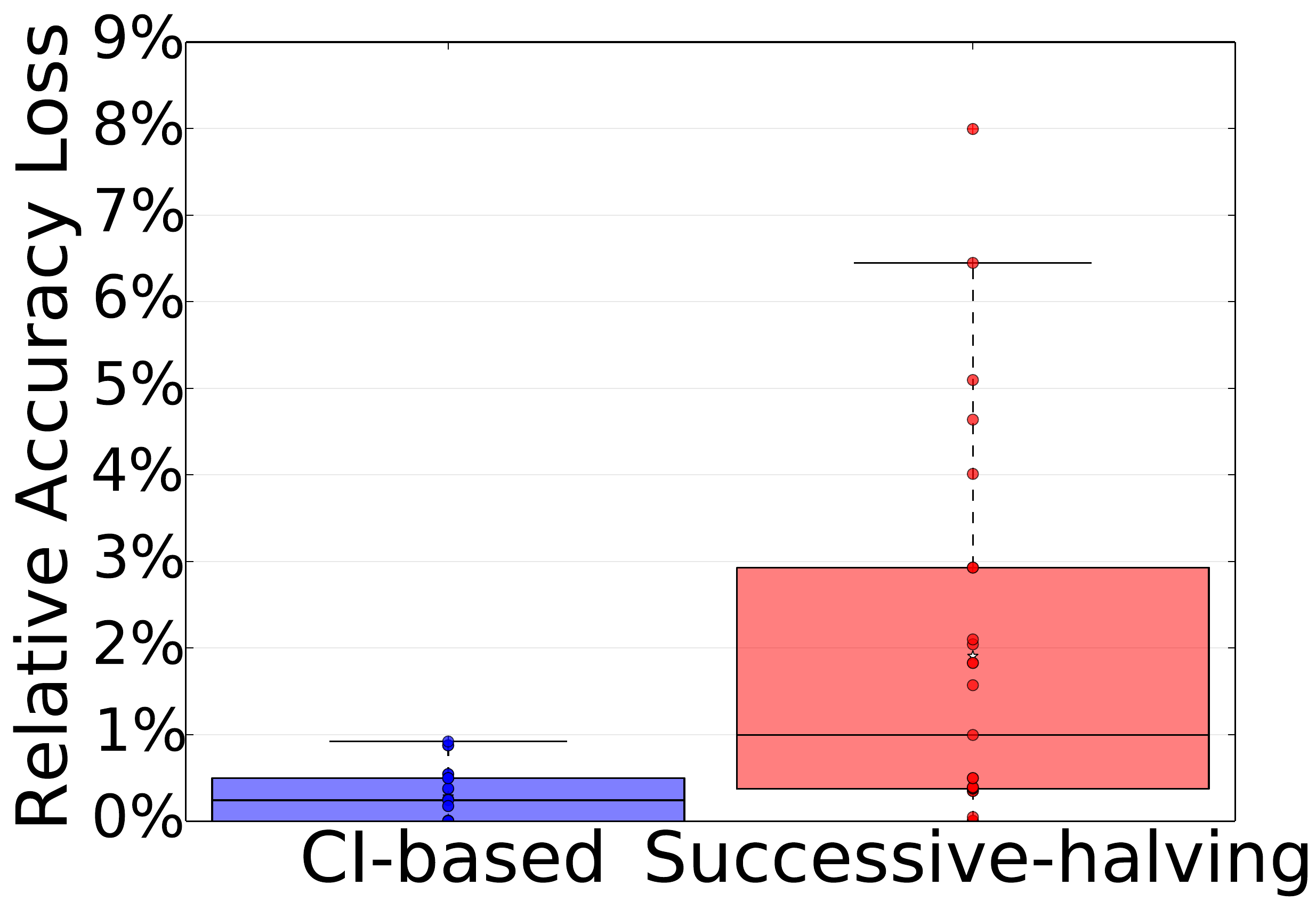}
        \caption{Boxplot of $\Delta_{rel}$}
        \label{fig:SH_percentage}
    \end{subfigure}
    \caption{\aml vs. Successive-halving} 
    \label{fig:SH_delta}
\end{figure}

We depict the comparison between Successive-halving and our \aml in Figure~\ref{fig:SH_delta}. The x-axis in Figure~\ref{fig:SH_scatter} refers to the relative accuracy loss compared to the best configuration by Full-run, the y-axis is the speedup compared to the running time of Full-run, and each point corresponds to a specific experiment with a certain dataset and $\cc$. We can see that Successive-halving has a similar speedup as \aml over Full-run. However, the relative accuracy loss can be an order of magnitude larger than that of \aml, e.g., $8\%$ vs. $0.8\%$. This is because the pruning performed in Successive-halving is based on the ranking of the current test accuracy. 
On the contrary, \aml uses confidence interval of the real test accuracy to perform safer pruning. 
Figure~\ref{fig:SH_percentage} presents a boxplot summarizing the relative accuracy loss for our solution and Successive-halving respectively. On average, the relative accuracy loss for our CI-based solution is $0.24\%$ (all below 1\%), and 2\% for Successive-halving (up to 8\%), which is nearly ten times larger.  
{\new In order to fully compare Successive-halving and our \aml, we have also conducted another set of experiments with varying time constraint. The results can be found in the Appendix. In addition, we report the performance of an enhanced algorithm of DAUB using our CIEstimator and CI-based pruning in the Appendix. }

\section{Related Work}\label{sec:related}
\stitle{AutoML.} AutoML has gained increasing attention in the past few years. 
The scope of AutoML includes automated feature engineering, model selection, and hyperparameter tuning process.
Some prevailing AutoML tools are Auto-sklearn for Python~\cite{NIPS:eurer2015} and Auto-Weka for Java~\cite{KDD:thornton2013auto}. 
Most research focus is devoted to the search strategy, i.e., which configurations to evaluate. The strategies can be broadly categorized as grid search~\cite{JMLR:Pedregosa2011}, random search~\cite{JMLR:bergstra2012random}, spectral search~\cite{ICLR:hazan2018hyperparameter}, Bayesian optimization~\cite{COLT:hutter2011sequential,NIPS:bergstra2011algorithms,NIPS:snoek2012practical}, meta-learning~\cite{NIPS:eurer2015}, and genetic programming~\cite{GECCO:Olson2016}. 
Few studies address the efficiency issue in ranking these configurations on large datasets. 
TuPAQ~\cite{SoCC:sparks2015automating} and HyperDrive~\cite{Middleware:rasley2017hyperdrive} are two systems which focus on hyperparameter tuning when all the configurations correspond to iterative training processes. They distribute the configurations into multiple machines, and use heuristic early stopping rules for training iterations.
 \cite{AISTATS:jamieson2016non} further models this problem as a non-stochastic multi-armed bandit process, where each arm corresponds to a configuration, each pull corresponds to a few training iterations, and the reward is the intermediate accuracy on the test data. Recognizing the difference with the stochastic bandit process, they propose a Successive-halving pruning strategy in the fixed budget setting. They focus on this setting because they have found it difficult to derive the confidence bounds of real test accuracy based on limited training iterations.
 Hyperband~\cite{ICLR:li2017hyperband} uses Successive-halving as a building block and tries to vary the number of random configurations under the same budget. While Hyperband suggests that the notion of resource can be generalized from training iterations to sample size of training data, we should notice that it is now possible to derive confidence bounds of real test accuracy based on sampled training data. Therefore, \aml can be used to replace Successive-halving in this scenario to achieve lower accuracy loss. In the Bayesian optimization framework, RoBO~\cite{bayesopt:klein17} treats the sample size as a hyperparameter, and uses random sample size to evaluate each configuration and a kernel function to extrapolate the real test accuracy.

\stitle{Generalization Error Bounds.} Generalization error bound has been studied extensively~\cite{JC:zhou2002covering,TAC:koltchinskii2000improved,JMLR:bousquet2002stability}, among which {\em VC-bound}~\cite{TNN:vapnik1999overview} is a well-known technique for bounding the generalization error. The main idea behind VC-bounds is to use {\em VC-dimension} to characterize the complexity of the hypothesis class. 
Besides VC-dimension, other existing techniques for deriving generalization error bounds include covering number~\cite{JC:zhou2002covering}, Rademacher complexity~\cite{TAC:koltchinskii2000improved}, and stability bound~\cite{JMLR:bousquet2002stability}. While the definition of generalization error bounds is different from the confidence bound needed for \aml, they have been used in other work to guide progressive sampling for a \emph{single} configuration~\cite{AAAI:elomaa2002progressive}.

\section{Discussion and Conclusion}
We studied the problem of efficiently finding approximate best configuration among a given set of training configurations for a large dataset. Our CI-based progressive sampling and pruning solution \aml can successfully select a top configuration with small or no accuracy loss, in much less time than the exact approach. The CI-based pruning is more robust than pruning based on point estimates.

There are multiple use cases that can benefit from our proposed \aml. The input of \aml can be either specified by the users based on their domain knowledge, or generated from an AutoML search algorithm. 
Our \aml module can help data scientists select a top configuration faster. As they iteratively refine it, they can use \aml to verify whether altering part of the configuration (such as changing features) boosts the performance, by invoking \aml with the old and new configurations. 
In addition, our confidence bounds can be potentially used to accelerate Bayesian optimization and spectral search in large datasets, which is interesting future work. 

{
\small
\bibliographystyle{aaai.bst}
\bibliography{all}
}

\appendix
\appendixpage

\section{Optimal Step Size in Geometric Scheduling} \label{sec:stepSize}
When $T_i(s)$ is a power function over the sample size, i.e., $T_i(s) = s^\alpha$ where $\alpha$ is a real number, the optimal step size follows $c = 2^{\frac{1}{\alpha}}$.
\begin{proof}
Assume $s_i^*$ is the {\em optimal} training sample size for each configuration $C_i$ when Algorithm~\ref{alg:CI-based} terminates.
As we do not know the optimal $s_i^*$ in advance, we try probes with progressive sample size for each configuration $C_i$, denoted as $\{s_i^1, s_i^2, \cdots s_i^m\}$, where $s_i^m$ is the sample size in the last probe of $C_i$ when Algorithm~\ref{alg:CI-based} terminates. For a fixed $i$, we assume that the last probe of all the other configurations uses their optimal training sample size, i.e., $s_{i'}^m=s_{i'}^*,\forall i'\neq i$. Thus, the termination point for $C_i$ must satisfy (a) $s_i^m \geq s_i^*$, otherwise $s_i^m$ would be a smaller sample size for configuration $C_i$ than $s_i^*$; and (b) $s_i^{m-1} < s_i^*$, otherwise Algorithm~\ref{alg:CI-based} would terminate at $s_i^{m-1}$ instead of $s_i^{m}$. 

Based on the above property, we first claim that with geometric schedule (i.e., $s_i^j=c^{j-1}s_i^1$, where $s_i^1$ is the sample size in the initial probe and $c$ is the geometric step size), the accumulated running time is asymptotically equivalent to the optimal running time (i.e., $\sum_{j=1}^m T_i(s_i^j) = \mathcal{O}(T_i(s_i^*))$)~\cite{KDD:provost1999}. Recall that $T_i(s) = s^\alpha$ is the probing time with a sampled training dataset of size $s$ for configuration $C_i$.

We further minimize the worst case ratio between the accumulated running time and the optimal running time, i.e., $\frac{\sum_{j=1}^m T_i(s_i^j)}{T_i(s_i^*)}$. Since $s_i^{m-1} < s_i^* \leq s_i^m$, the worst case occurs when $s_i^* = s_i^{m-1}$. By replacing each term $T_i(s)$ with $s^\alpha$ and solving $\min_c \frac{\sum_{j=1}^m T_i(s_i^j)}{T_i(s_i^{m-1})}$, we can derive a closed form solution for $c$. When the step size follows $c = 2^{\frac{1}{\alpha}}$, $\frac{\sum_{j=1}^m T_i(s_j)}{T_i(s_i^{m-1})}$ is minimized and it is guaranteed that $\frac{\sum_{j=1}^m T_i(s_j)}{T_i(s_i^*)} \leq 4$ in any case.
For instance, when the training time $T_i(s)$ increases linearly with the sample size $s$, i.e., $\alpha = 1$, then we should set $c$ to 2, which means we should double the sample size as the probing proceeds for each configuration. In fact, Provost et.al.~\cite{KDD:provost1999} set $c=2$ heuristically and found good empirical performance. Our analysis provides a theoretical justification for that heuristic.
\end{proof}

\section{Performance Analysis of \grad} \label{sec:grad_proof}
{\new
It is hard to bound the performance of \grad in general cases. However, under certain conditions, we are able to prove a 4-approx guarantee for \grad with respect to the oracle optimal running time when $\epsilon=0$ for the scenario where users need to get the trained classifier corresponding to the selected configuration $C_{i'}$. Remember that in this scenario, the total running time involves both the time to select the approximate best configuration, and the time taken to train the classifier on full data with $C_{i'}$. 
\begin{assumption} [CI Condition] \label{ass:CI}
All confidence intervals correctly bound the test accuracy and shrink as Algorithm~\ref{alg:CI-based} proceeds.
\end{assumption}
\begin{assumption} [Convex Condition] \label{ass:convex}
$f_i$ and $g_i$ are convex functions of $l_i$ and $u_i$ respectively.
\end{assumption}
\begin{assumption} [Proxy Condition] \label{ass:proxy}
For $C_i \in \cc$, $\frac{\Delta f_i}{\Delta l_i}=\frac{\Delta T_i}{\Delta l_i}$ and $\frac{\Delta g_i}{\Delta u_i}=\frac{\Delta T_i}{\Delta u_i}$.
\end{assumption}
\noindent First, CI Condition means that we only focus on the case where all the confidence intervals correctly bound the test accuracies. In this case, for each configuration, its lowest upper bound and highest lower bound in all the rounds still bound the real test accuracy correctly, and the upper (lower, resp.) bound is monotonically decreasing (increasing, resp.) as Algorithm~\ref{alg:CI-based} runs.
Second, the convex condition means that the increase of sample size has a diminishing return to the change of the CI -- as Algorithm~\ref{alg:CI-based} proceeds, it takes longer to attain the same increase (decrease, resp.) on the lower bound (upper bound, resp.). Last, since the gradient $\frac{\Delta f_i}{\Delta l_i}$ ($\frac{\Delta g_i}{\Delta u_i}$, resp.) is impossible to compute, we can only use $\frac{\Delta T_i}{\Delta l_i}$ ($\frac{\Delta T_i}{\Delta u_i}$, resp.) to approximate it. The proxy condition means such approximation is accurate and effective.

\begin{theorem}[\grad 4-Approx]\label{thn:grad}
Under Assumption~\ref{ass:CI} to \ref{ass:proxy}, \grad provides a 4-approx guarantee to the oracle optimal running time to get the final trained classifier when $\epsilon=0$\eat{, with probability at least $1-\delta$}.
\end{theorem}
}

\begin{proof}
\eat{We focus on the case where all the confidence intervals correctly bound the test accuracies, which occurs with probability at least $1-\delta$. In this case, for each configuration, its lowest upper bound and highest lower bound in all the rounds still bound the real test accuracy correctly. So without loss of generality, we can assume that with the geometric increase of the sample size, the confidence interval of each configuration shrinks as Algorithm~\ref{alg:CI-based} proceeds. That is, for each configuration the upper (lower) bound is monotonically decreasing (increasing) as Algorithm~\ref{alg:CI-based} runs. }
We illustrate our analysis with the help of Figure~\ref{fig:alg_analysis}. Without loss of generality, we consider $C_1$ as the best configuration $C_{i^*}$. Each vertical line corresponds to one configuration, with shrinking confidence intervals as Algorithm~\ref{alg:CI-based} proceeds {\new according to the CI condition}. 
The red and blue lines with the same length depict a corresponding pair of upper and lower bound after a particular probe in Algorithm~\ref{alg:CI-based}. Let $l^*=u^*$ be the optimal solution in Equation~\eqref{eqn:constraints}, as shown by the solid horizontal black line in Figure~\ref{fig:alg_analysis}. Furthermore, as shown in Figure~\ref{fig:alg_analysis}, let $l_1^*$ be the smallest lower bound of $C_1$ that is {\new larger}\eat{no less} than $l^*$, and $u_i^*$ be the largest upper bound of $C_i$ that is {\new smaller}\eat{no larger} than $u^*$, where $2\leq i\leq n$. In the following, we prove that with our proposed scheduling scheme \grad, for each configuration $C_i$, $2\leq i \leq n$, the upper confidence bound cannot cross below the red solid line $u_i^*$ in Figure~\ref{fig:alg_analysis}, i.e., Algorithm~\ref{alg:CI-based} terminates with $u_i \geq u_i^*$. 

\begin{figure}[h]
\centering
\includegraphics[height = 4.2cm, width=7cm]{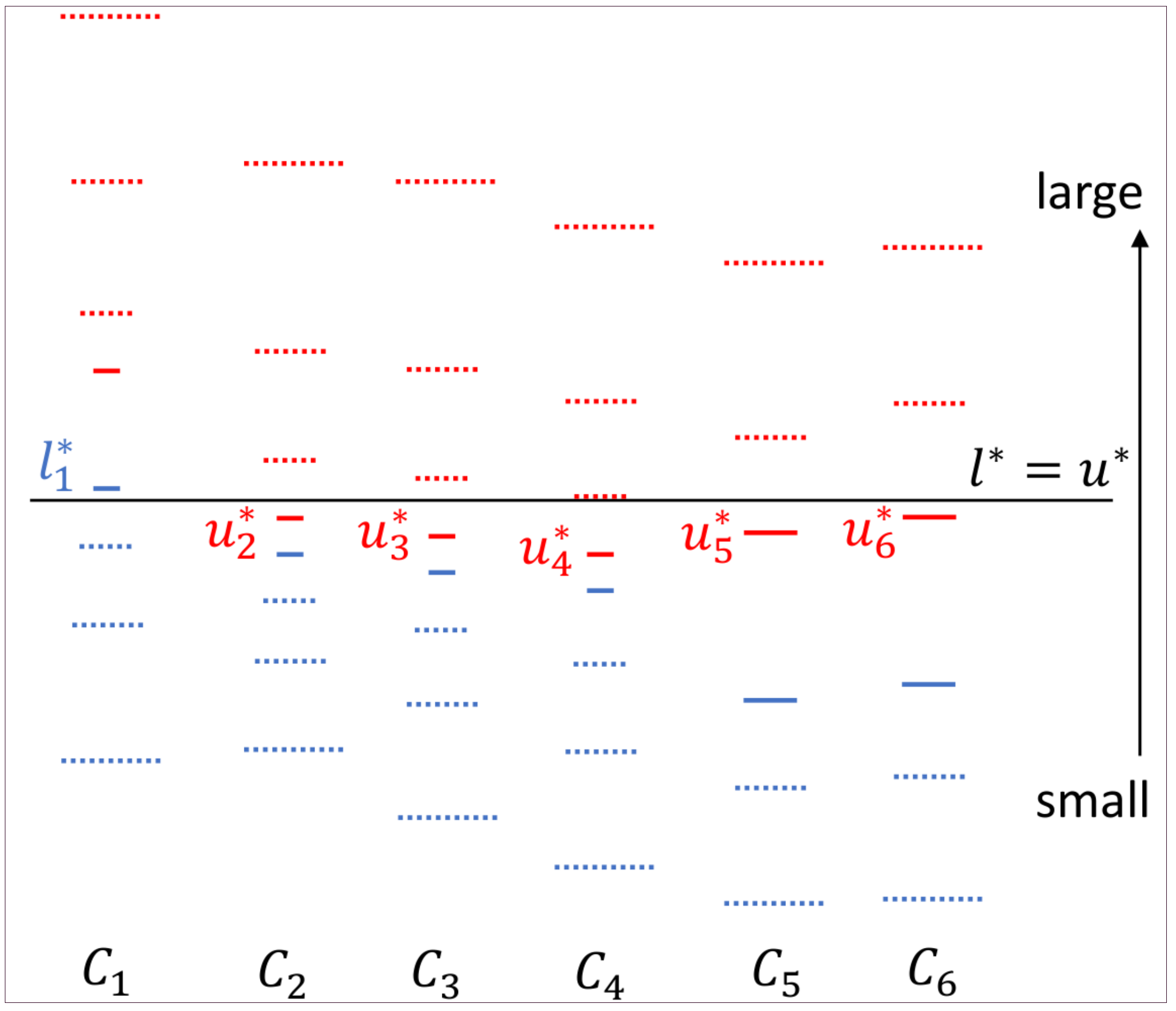}
\caption{Analysis of \grad}
\label{fig:alg_analysis}
\end{figure}

For any $2\leq i \leq n$, we will prove {\em by induction} that $u_i \geq u_i^*$ at every iteration of Algorithm~\ref{alg:CI-based}. First, it is obvious that in the first iteration (or probe), $u_i \geq u_i^*$ for any $2\leq i \leq n$. Next, suppose $u_i \geq u_i^*$ after the $(k-1)^{th}$ iteration of Algorithm~\ref{alg:CI-based}, we will show that $u_i \geq u_i^*$ after the $k^{th}$ iteration for any $2\leq i \leq n$ and $k \geq 2$. Since one probe is performed in each iteration, we only need to prove that for the probing configuration $C_{prob}$, $u_{prob} \geq u_{prob}^*$ still holds when $C_{prob} \neq C_1$. 
Recall that in each iteration, \grad makes the guess that the configuration with largest upper bound is the best configuration. In the following, we discuss two cases depending on whether that guess is right in the $k^{th}$ iteration. For notation simplicity, let $C_j$ be the probing configuration $C_{prob}$ in the $k^{th}$ iteration. 

\stitle{Case 1.} \grad has a wrong guess of the best configuration $C_1$. Suppose $C_{i'}$ is speculated as the best configuration. By definition, $u_{i'} \geq u_1 \geq \calA_1 \geq l^* > u_{i'}^*$. Hence, if the probing configuration $C_{j}$ is $C_{i'}$, then we have shown that $u_{j} > u_{j}^*$. Otherwise, the probing configuration $C_j$ must be the one with the second highest upper bound, according to \grad. In that case, $u_j \geq u_1$ since $C_1$ is also compared against when identifying the configuration with second highest upper bound, and  $u_1 \geq \calA_1 \geq l^* > u_{j}^*$ by definition. Thus, $u_j > u_{j}^*$. {\new Now we have shown that $u_j > u_j^*$ at the beginning of the $k^{th}$ iteration. Thus, we have $u_j \geq u_j^*$ after the $k^{th}$ iteration due to the one probe on $C_j$.}

\stitle{Case 2.} \grad has a correct guess on $C_1$. If $C_1$ is probed, then we are done since $u_i$ does not change for $2\leq i \leq n$. Otherwise, the probing configuration $C_j$ is the one with the second highest upper bound, according to \grad. In that case, we prove $u_j \geq u_j^*$ after the $k^{th}$ iteration {\em by contradiction}. If $u_j < u_j^*$ after the $k^{th}$ iteration, then $u_j$ must equal $u_j^*$ at the beginning of the $k^{th}$ iteration since we assume $u_j \geq u_j^*$ holds after the $(k-1)^{th}$ iteration. Next, we will show $u_i = u_i^*$ at the beginning of the $k^{th}$ iteration for all $2\leq i \leq n$ {\new and $i\neq j$}.

First, for $2\leq i \leq n$, $i\neq j$, and $u_i \in \Omega$ at the beginning of the $k^{th}$ iteration in Algorithm~\ref{alg:CI-based}, $u_i$ must equal $u_i^*$, since otherwise $C_j$ can not be with the second highest upper bound. Recall the $\Omega$ is the remaining candidate configurations. Second, for $2\leq i \leq n$, $i\neq j$, and $u_i \not\in \Omega$, we will show $C_i$ must have been pruned when $u_i$ equals $u_i^*$, by contradiction. Otherwise (i.e., $u_i>u_i^*$), $u_i \geq u^*$ and $C_j$ is thus pruned before the $k^{th}$ iteration since $u_j = u_j^* < u^* \leq u_i$, which contradict with the fact that $C_j \in \Omega$. 

Hence, we have $u_i = u_i^*$ at the beginning of the $k^{th}$ iteration for all $2\leq i \leq n$. Thus, $l_1 < l^*$, since otherwise all the other configurations are pruned. As a consequence, we have $\frac{df_1}{dl_1} \leq |\frac{dg_2}{du_2}+\cdots+\frac{dg_{n}}{du_n}|$, since $\frac{df_1}{dl_1}$ decreases with the decrease of $l_1$ and $|\frac{dg_i}{du_i}|$ increases with the decrease of $u_i$ according to the convex condition. Then based on \grad, $C_1$ should be probed, which contradicts with the assumption that the probing configuration $C_{j}$ is with the second highest upper bound. 
Altogether, $u_j \geq u_{j}^*$ for case 2.


Combining case 1 and case 2, we now have shown that when Algorithm~\ref{alg:CI-based} terminates, $u_i \geq u_i^*$, $\forall 2\leq i\leq n$. Equivalently speaking, each configuration $C_i$ ($2\leq i \leq n$) is probed at most one more time compared to the optimal scheme (the solid horizontal black line in Figure~\ref{fig:alg_analysis}). In addition, given a configuration, each probe's running time is twice of that in its previous probe, since we have set $c^{\alpha} = 2$.
Thus, the accumulated running time of \grad is at most 4 times of the optimal runtime for any configuration $C_i$, where $2\leq i \leq n$. That is:
\[t_i \leq 4 {t_i^*}, \forall 2\leq i \leq n\]
where $t_i^*$ is the optimal running time corresponding to the optimal scheme for identifying the best configuration and $t_i$ is the accumulated running time for $C_i$ in \grad. In the worst case, $C_1$ is probed all the way till with full training data.
With $c^\alpha = 2$, we have $t_1 \leq 2 T_1(|\dd_{tr}|)$.

{\new Recall that in the scenario where users need to get the final trained classifier, the total running time has two terms: the time taken to select the approximate best configuration and the time to get the trained classifier corresponding to the selected configuration (i.e., $\calT = \sum_i {t}_i$ + $T_1(|\dd_{tr}|)$).} Since $\calT_{min} = \sum_{i=1}^n{t_i^*} + T_1(|\dd_{tr}|)$ and $\calT = \sum_{i=1}^n{t_i} + T_1(|\dd_{tr}|)$, we have $\calT \leq 4 \calT_{min}$.
\end{proof}

{\new
\section{\ucb as Scheduler} \label{sec:gucb_proof}
{Upper Confidence Bound (UCB)} is a common scheduling scheme in multi-armed bandit problem. To apply UCB to our problem, in each iteration we always pick the configuration with the highest upper confidence bound for probing. The intuition is that the upper bound reflects the potential of this particular configuration, and thus the configuration with higher upper bound deserves more exploration. 
UCB tries to push down the highest upper bound, which kind of matches the second condition in Equation~\eqref{eqn:constraints}. However, UCB does not take lower bound's growth rate and upper bound's decrease rate into consideration, i.e., the first condition in Equation~\eqref{eqn:constraints}. Intuitively, when the configuration with the highest upper bound takes a long time for probing but the confidence interval shrinkage is only marginal, we should opt for probing other configurations.
Motivated by this, our proposed \scheduler, i.e., \grad, have taken the potential (i.e., upper bound), together with the upper bound's decrease rate and the lower bound's growth rate into consideration.

To compare the gradient-based approach (i.e., \grad) with \ucb, we also apply geometric scheduling within each configuration for \ucb. Under certain conditions, we are able to show a $4\times$ approximation guarantee for \ucb in Theorem~\ref{thn:ucb} when $\epsilon=0$.

\eat{To compare the gradient-based approach with \ucb, we combine the geometric probing in each configuration with \ucb across different configurations, titled \gucb. Under certain conditions, we show a $4\times$ approximation guarantee for \ucb in Theorem~\ref{thn:ucb} when $\epsilon=0$.}

\eat{Combining the geometric scheduling in each configuration with \ucb across different configurations respectively, we obtain a scheduling scheme for \scheduler, titled \gucb. We show the approximation guarantee for \gucb in Theorem~\ref{thn:ucb}. In particular, when $\epsilon=0$, the run time of \gucb is within $4\times$ to the optimal total runtime with high probability. }

\begin{theorem}[\ucb Approx Guarantee]\label{thn:ucb}
Under Assumption~\ref{ass:CI}, \ucb provides a $4$-approx guarantee to the optimal running time to get the final trained
classifier when $\epsilon=0$.
\end{theorem}

\begin{proof}
Similar to the proof in Theorem~\ref{thn:grad}, we use Figure~\ref{fig:alg_analysis} to help illustrate the analysis. In the following, we prove that with \ucb, for each configuration $C_i$, $2\leq i \leq n$, the upper confidence bound cannot cross below the red solid line $u_i^*$ in Figure~\ref{fig:alg_analysis}, i.e., Algorithm~\ref{alg:CI-based} terminates with $u_i \geq u_i^*$.

For any $2\leq i \leq n$, we will prove {\em by induction} that $u_i \geq u_i^*$ at every iteration of Algorithm~\ref{alg:CI-based}. First, it is obvious that in the first iteration (or probe), $u_i \geq u_i^*$ for any $2\leq i \leq n$. Next, suppose $u_i \geq u_i^*$ after the $(k-1)^{th}$ iteration of Algorithm~\ref{alg:CI-based}, we will show that $u_i \geq u_i^*$ after the $k^{th}$ iteration for any $2\leq i \leq n$ and $k \geq 2$. Since one probe is performed in each iteration, we only need to prove that for the probing configuration $C_{prob}$, $u_{prob} \geq u_{prob}^*$ still holds when $C_{prob} \neq C_1$. For notation simplicity, let $C_j$ be the probing configuration $C_{prob}$ in the $k^{th}$ iteration.

We prove it {\em by contradiction}. If $u_{j} < u^*$ after the $k^{th}$ iteration. then $u_j$ must equal $u_j^*$ at the beginning of the $k^{th}$ iteration since we assume $u_j \geq u_j^*$ holds after the $(k-1)^{th}$ iteration. That is, $u_j = u_j^* < u^*$. Based on the confidence interval property, we know that $u^* = l^* \leq \calA_1 \leq u_1$. Hence, we have $u_{j} < u_1$, which contradicts with the fact that $C_{j}$ is with the highest upper bound and is selected as the probe configuration. 

Thus, we proved that each configuration probes at most one more time than the optimal scheme, i.e., $u_i \geq u_i^*$ where $2\leq i \leq n$. The remainder of the proof is the same as in Theorem~\ref{thn:grad}.
\eat{Furthermore, for a specific configuration, each probe takes twice of the time in its previous probe. Thus, We have the following inequality, where $t_i^*$ is the one-probe runtime for $C_i$ in the optimal scheme.
$$\tilde{t}_i \leq 4 {t_i^*}, \forall 2\leq i \leq n$$

In the worst case, $C_1$ is probed all the way till with full training data. Recall that $c^\alpha = 2$, thus,
$$\tilde{t}_1 \leq 2 T_1(\dd_{tr})$$

Since $\calT_{min} = \sum_{i=1}^n{t_i^*} + T_1(\dd_{tr})$, and $\calT = \sum_{i=1}^n{\tilde{t}_i} + T_1(\dd_{tr})$, we have:
$$\Rightarrow \calT \leq 4 \calT_{min}$$
}
\end{proof}
}

\section{Extra Experiments}\label{sec:extraExp}
\subsection{Varying {\Large $\epsilon$} in CI-based Framework}\label{ssec:exp_epsilon}
\begin{figure}[h!]
\centering
    \vspace{-2mm}
    \begin{subfigure}{0.49\textwidth}
        \centering
        \includegraphics[width=\linewidth]{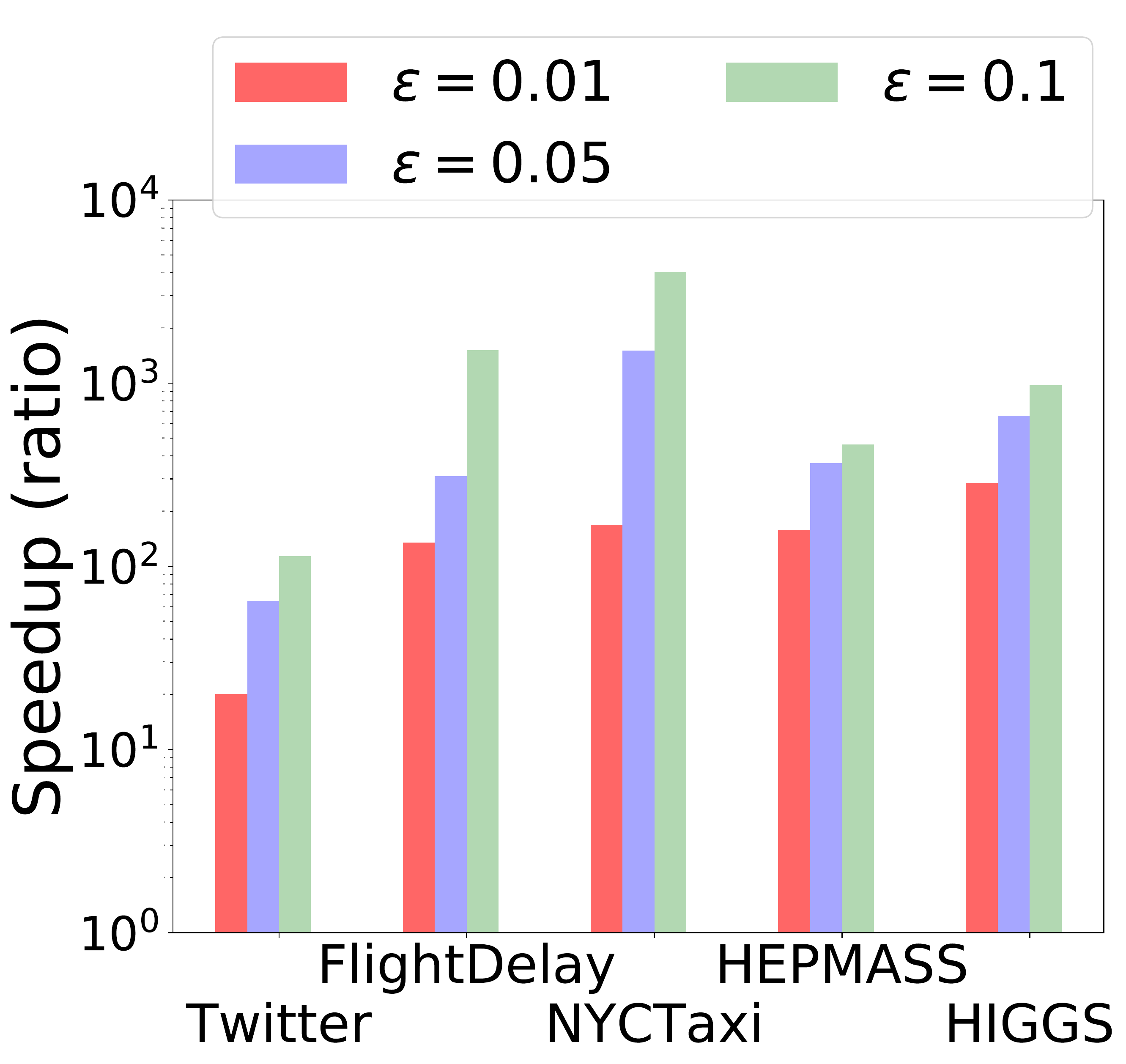}
        \vspace{-5mm}
        \caption{Speedup} 
        \label{fig:epsilon_speedup}
    \end{subfigure}
    \vspace{5mm}
    \begin{subfigure}{.49\textwidth}
        \centering
        \includegraphics[width=\linewidth]{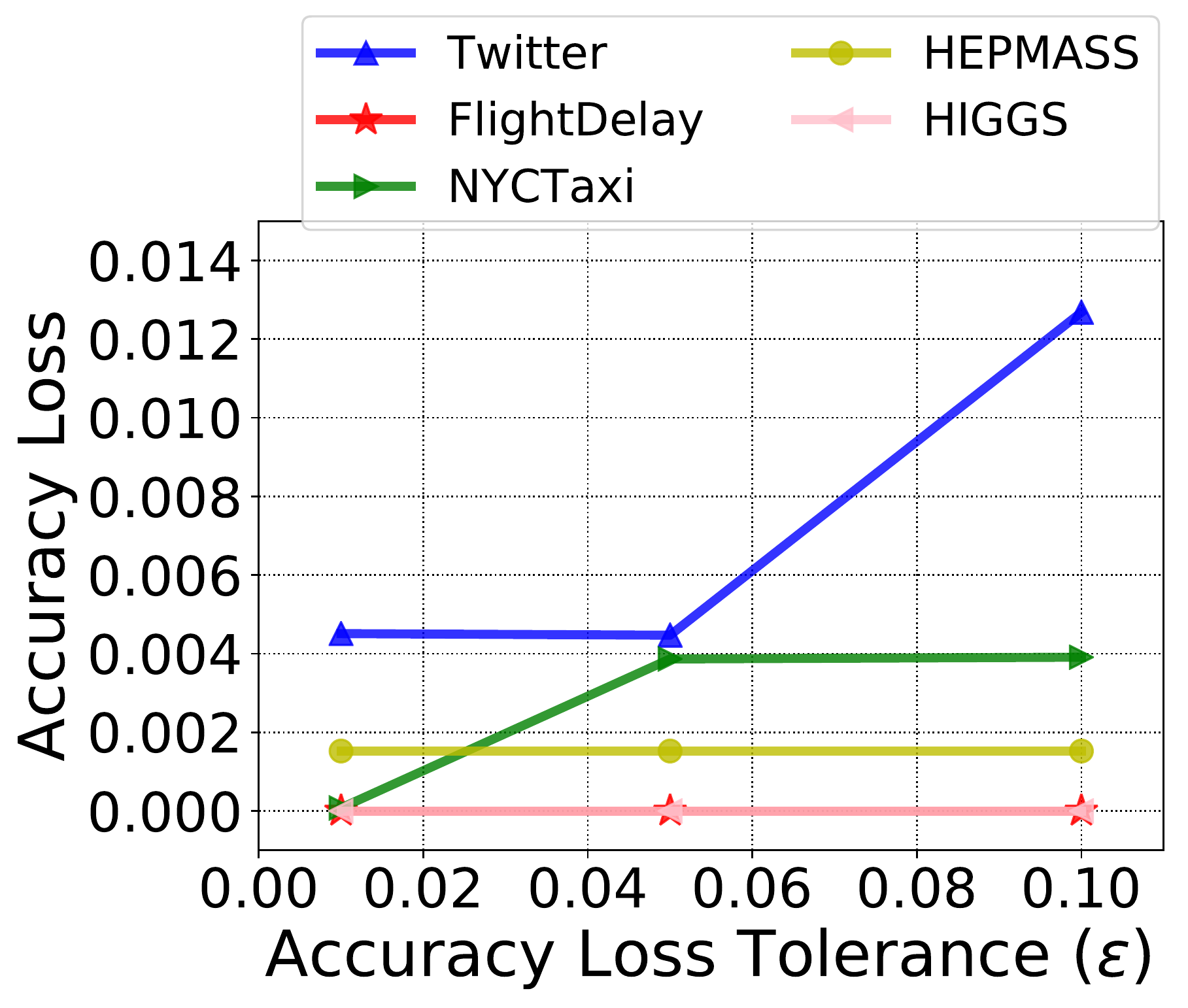}
        \vspace{-3mm}
        \caption{Accuracy Loss} 
        \label{fig:epsilon_accuracy}
    \end{subfigure}
    \vspace{-5mm}
    \caption{Speedup vs. Accuracy Loss with Varying $\epsilon$}
    \label{fig:varyingEps}
\end{figure}
In this experiment, we compare the real accuracy loss and the speedup with varying input accuracy loss tolerance $\epsilon$. As shown in Figure~\ref{fig:varyingEps}, with the increase of the input $\epsilon$, both the speedup and the real accuracy loss increase. First, as $\epsilon$ increases, the pruning condition is easier to be satisfied, leading to faster termination of Algorithm~\ref{alg:CI-based}. Second, with smaller running time (i.e., resource) on each configuration, the \plateau tends to be less accurate. Consequently, the real accuracy loss typically increases as the increase of $\epsilon$. As depicted in Figure~\ref{fig:epsilon_accuracy}, the real accuracy loss increases slightly as $\epsilon$ increases. Specifically, when $\epsilon=0.1$, the real accuracy loss is 0.012 for Twitter, and below 0.004 for other datasets, while the speedup reaches 3 orders of magnitude for FlightDelay, NYCTaxi, and HIGGS.

\subsection{Varying Time Constraint in CI-based Pruning vs. Successive-halving}\label{ssec:sh_b}
\begin{figure}[h]
    \centering
        \vspace{-2mm}
        \begin{subfigure}{0.49\textwidth}
            \centering
            \includegraphics[width=\linewidth]{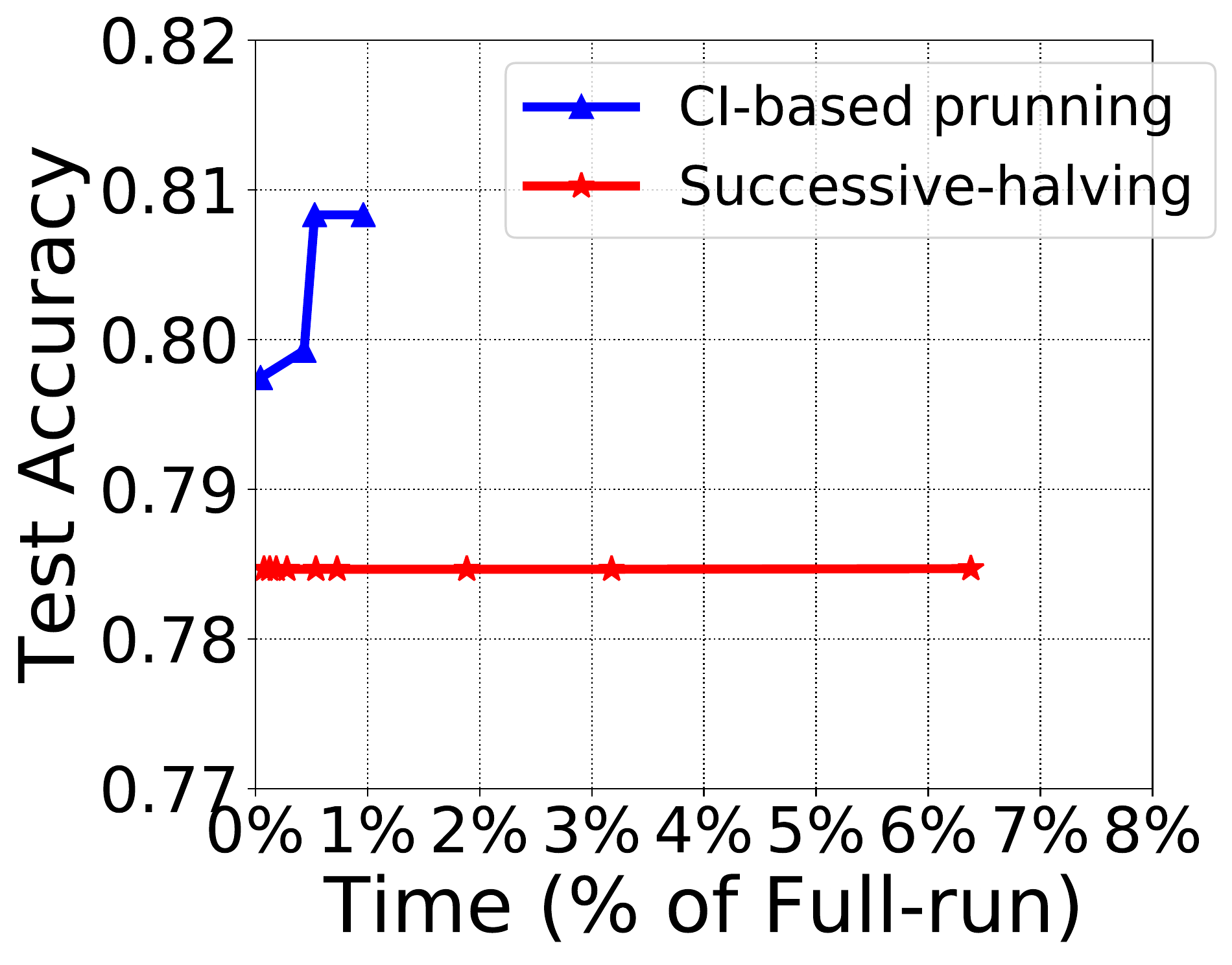}
            \vspace{-5mm}
            \caption{FlightDelay}
            \label{fig:SH_flight}
        \end{subfigure}
        \begin{subfigure}{.49\textwidth}
            \centering
            \includegraphics[width=\linewidth]{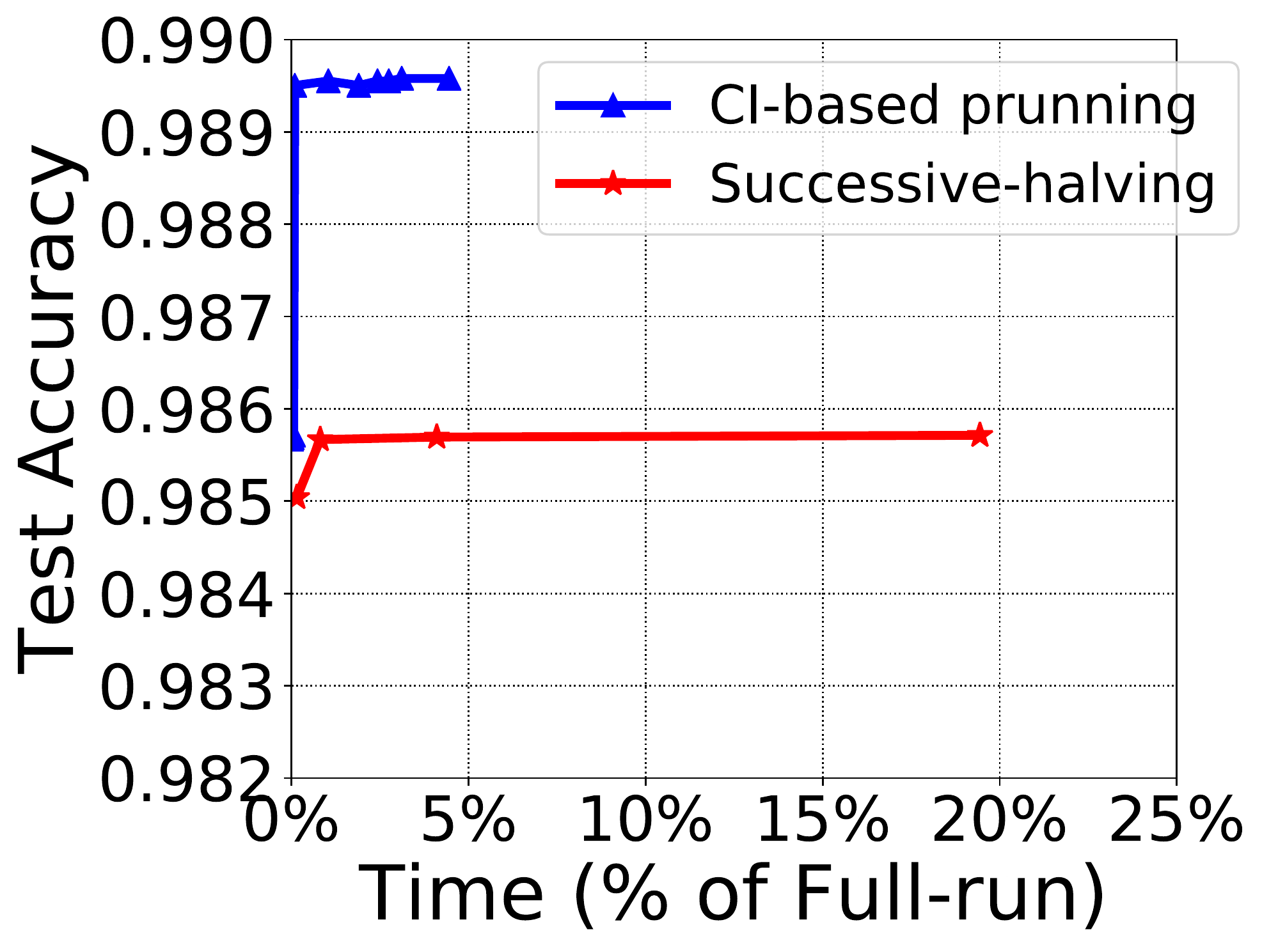}
            \vspace{-5mm}
            \caption{NYCTaxi}
            \label{fig:SH_NYCTaxi}
        \end{subfigure}
        \begin{subfigure}{0.49\textwidth}
            \centering
            \includegraphics[width=\linewidth]{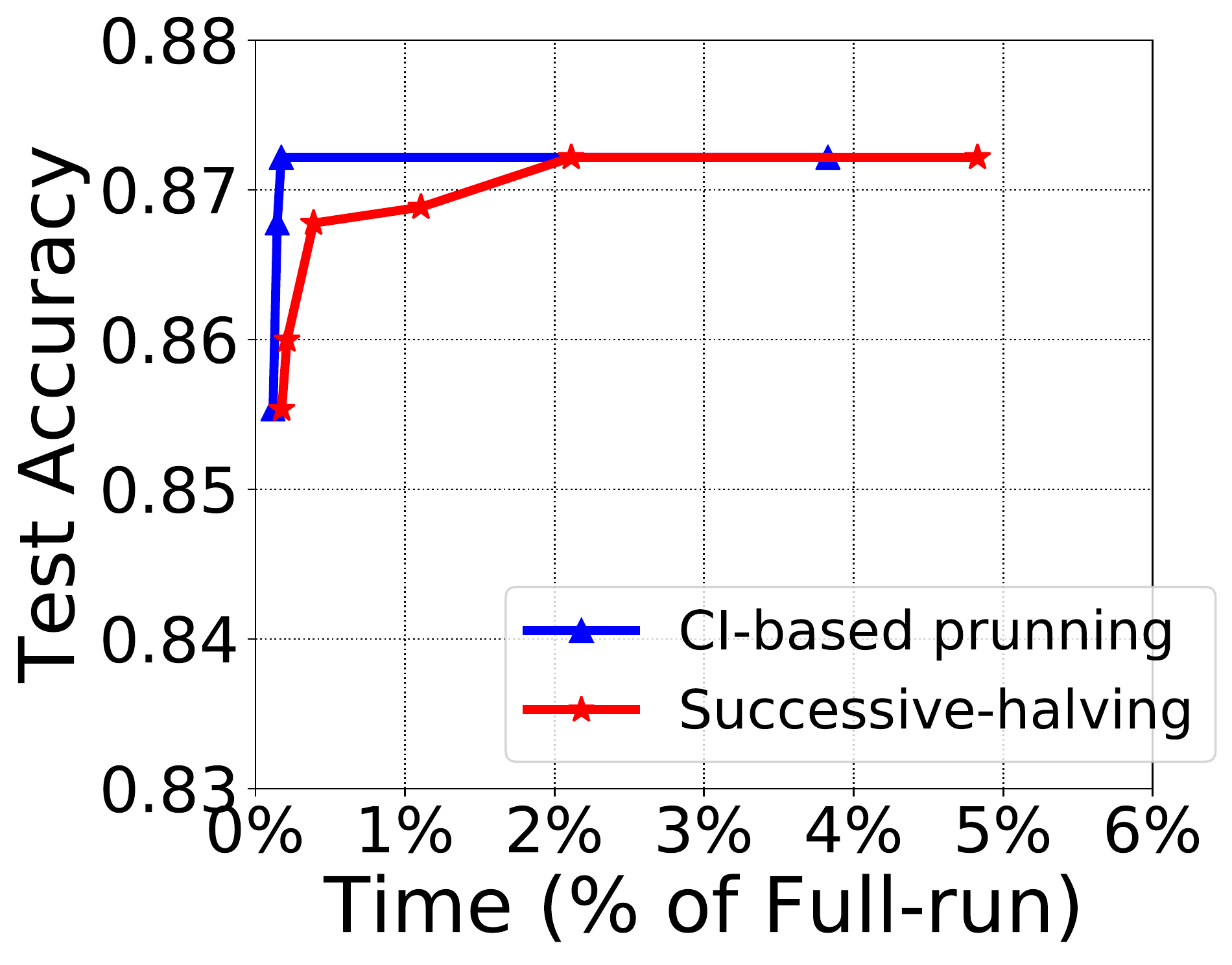}
            \vspace{-5mm}
            \caption{HEPMASS}
            \label{fig:SH_Hepmass}
        \end{subfigure}
        \begin{subfigure}{.49\textwidth}
            \centering
            \includegraphics[width=\linewidth]{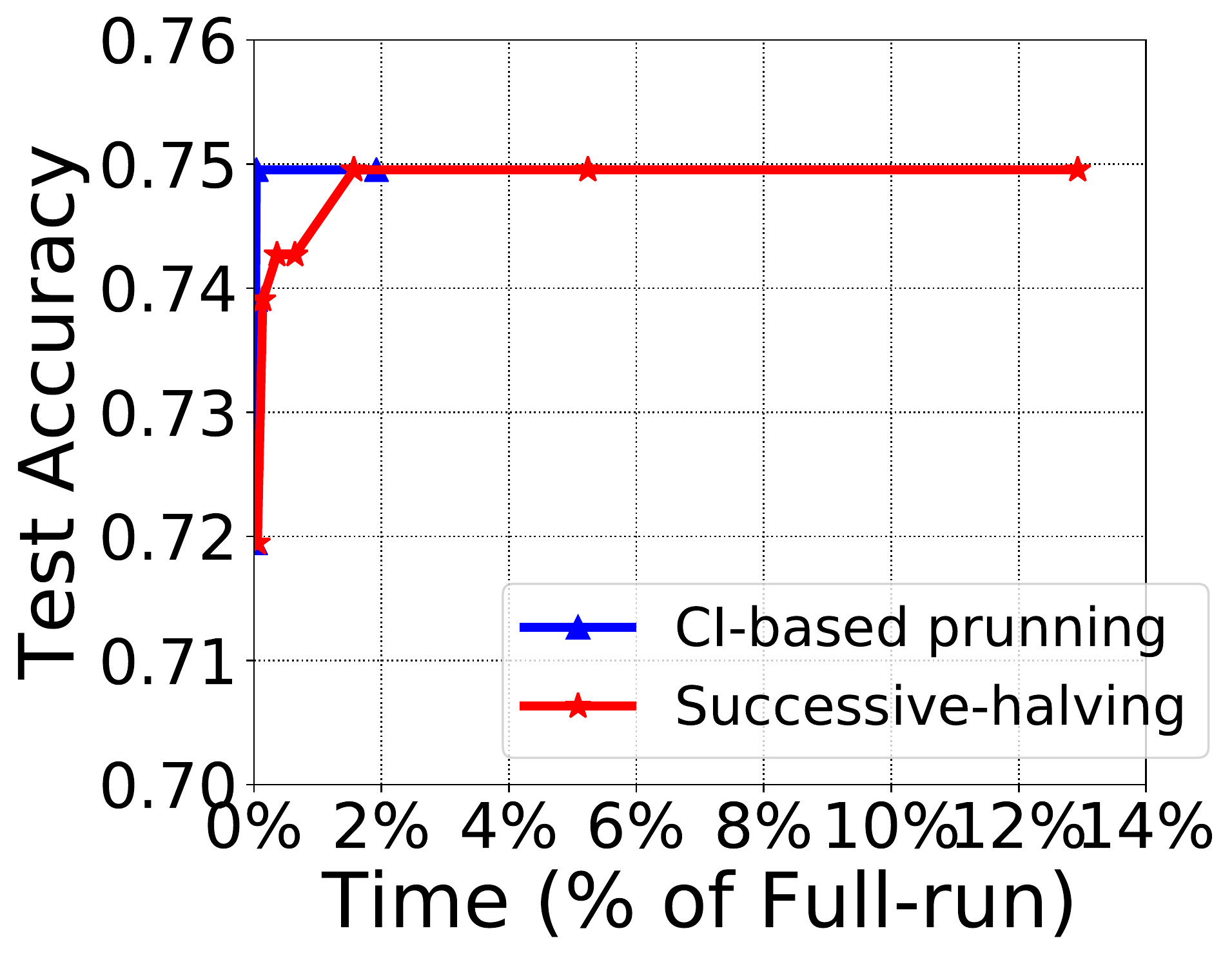}
            \vspace{-5mm}
            \caption{HIGGS}
            \label{fig:SH_Higgs}
        \end{subfigure}
        \vspace{-3mm}
        \caption{CI-based Pruning vs. Successive-halving with Varying Time Constraint\eat{in Scenario (b)}}
        \label{fig:SH_time}
    \end{figure}
    
In addition to the scenario\eat{scenario (a)} of no resource constraint in the main paper, we also study the performance of our CI-based framework and Successive-halving, when we impose resource constraint on these two algorithms. In general, our CI-based framework dominates Successive-halving with varying resources, i.e., with the same resource, our framework returns the configuration with higher testing accuracy than that provided by Successive-halving.

For Successive-halving, the resource budget is controlled by the initial sample size. We vary the initial training sample size in Successive-halving starting from 250 and increase the initial training sample size by 2 every time. The initial test sample size is always twice of the initial training sample size. Each point in the red line of Figure~\ref{fig:SH_time} corresponds to one initial sample size, and the left most point has 250 initial training samples. As a result, we can attain various running time of Successive-halving, and we further normalize them as the percentage of Full-run's running time. 
For our CI-based framework, we add the option for it to terminate at any iteration. In Algorithm~\ref{alg:CI-based}, we output a best-guess configuration at the end of each iteration. Specifically, we compare the configuration $C_{i'}$ (with the highest lower bound so far) against the configuration $\Omega_1$ (with the highest upper bound so far), and output the one with smaller gap between its lower bound and the highest upper bound of all the other configurations. The intuition is that, in order to prune all the other configurations, we need to compare the lower bound of the output configuration with the upper bound of all the other configurations, and we would like this gap to be as small as possible such that the bound of the accuracy loss between the output configuration and the best configuration is small. Each point in the blue line of Figure~\ref{fig:SH_time} corresponds to one such best-guess configuration.

In Figure~\ref{fig:SH_time}, x-axis is the running time percentage taken compared to Full-run, and y-axis is the real test accuracy for the returned configuration. From Figure~\ref{fig:SH_time}, we can see that our CI-based framework dominates Successive-halving. In particular, the test accuracy provided by Successive-halving is much worse than that returned by our framework in Figure~\ref{fig:SH_time}(a)(b); while in Figure~\ref{fig:SH_time}(c)(d) Successive-halving takes much longer time to reach the same test accuracy as that in our framework. In general, when starting with larger initial training sample size, Successive-halving can return better configuration. This is because the point estimation in Successive-halving is more accurate when using larger training sample size. However, this in turn increases the total running time of Successive-halving, making it inferior to our CI-based framework. The result suggests that the \aml is also useful in the resource-constrained scenario, though it was not designed for that scenario.  

{

\subsection{Comparison of Scheduling Schemes}\label{ssec:exp_scheduling}
\begin{figure}[t!]
        \centering
        \includegraphics[width=0.7\linewidth]{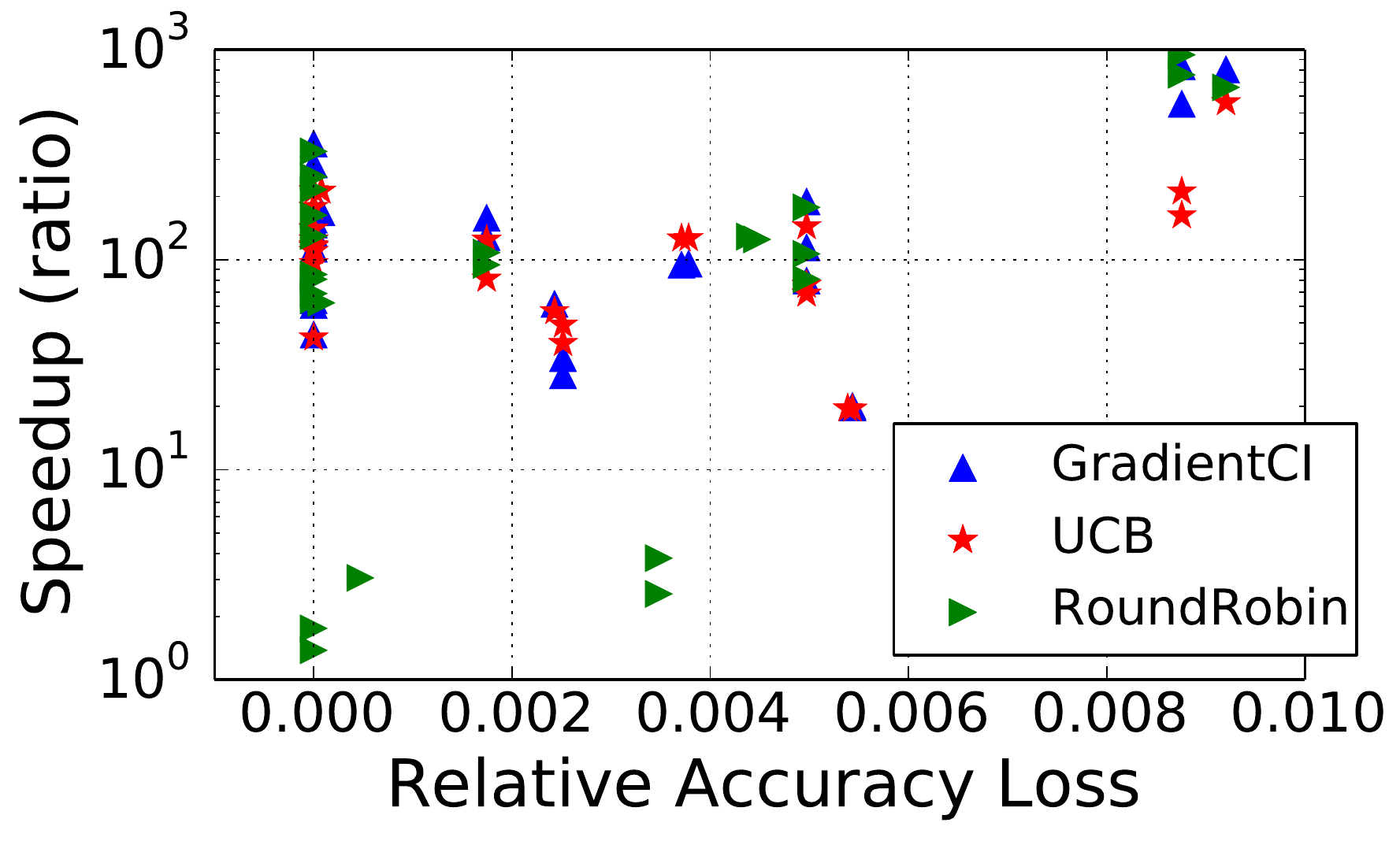}
    \caption{Comparison of Different Scheduling Schemes}
    \label{fig:scheduling}
\end{figure}
    


Within the CI-based framework \aml, we now empirically study the impact of three different scheduling schemes, i.e., \grad, {\em \ucb}, and {\sc RoundRobin}. 

\eat{\ucb is a widely used scheduling scheme in multi-arm bandit problem. As indicated by the name, in each iteration \ucb always picks the configuration with the highest upper confidence bound for probing. The intuition is that the upper bound reflects the potential of this particular configuration, and thus the configuration with higher upper bound deserves more exploration. 
To some extent, UCB pushes the upper bound of each configuration to end up with the same value, which kind of matches the second condition in Equation~\eqref{eqn:constraints}. However, UCB does not take lower bound's growth rate and upper bound's decrease rate into consideration, i.e., the first condition in Equation~\eqref{eqn:constraints}. 
}

{\new Recall that \ucb always picks the configuration with the highest upper confidence bound for probing.}
{\sc RoundRobin} allocates resources (i.e., probes) evenly among the remaining configurations. Specifically, {\sc RoundRobin} chooses the configuration with the smallest number of probes as the $C_{prob}$ in each iteration, replacing line 10 in Algorithm~\ref{alg:CI-based}.

We perform the same set of experiments as that in the main experiment section, but with different scheduling schemes in our CI-based framework.
First, we observe that in most cases, when applying different scheduling schemes, the accuracy loss of the returned configuration remains almost the same. However, the speedup differs from \grad to \ucb and {\sc RoundRobin}. Figure~\ref{fig:scheduling} depicts the speedup and relative accuracy loss achieved by different scheduling schemes. Each point refers to a particular dataset and a configuration set size $|\cc|$. With the same relative accuracy loss, \grad can achieve the highest speedup in most cases. We notice that {\sc RoundRobin} performs much more slowly in a few cases (speedup below 5 while the other two schedulers achieve over 20$\times$ speedup). This implies the non-adaptive scheduling can waste resources. \grad and \ucb are more robust. The average speedup of \grad and \ucb are 190$\times$ and 128$\times$ respectively. That shows the benefit of taking the speed of CI change into consideration during scheduling.

Note that the \ucb method evaluated in this section is an enhanced algorithm of DAUB~\cite{AAAI:Sabharwal16}. Both \ucb and DAUB use the same scheduling scheme, but \ucb uses our novel CIEstimator and CI-based pruning technique to ensure the $\epsilon$-guarantee. 
}



\end{document}